\documentclass[11pt]{article}
\usepackage{hyperref}
\usepackage{fullpage}
\usepackage{authblk}

\usepackage{macros}
\usepackage{slivkins-setup}
\usepackage{natbib}
\usepackage{amsmath}
\usepackage{multicol}
\usepackage{enumitem}

\renewcommand{\xhdr}[1]{\vspace{2mm} \noindent{\bf #1}}

\usepackage{color-edits}
\addauthor{ka}{black}   
\addauthor{as}{black}      
\addauthor{bug}{black} 

\newcommand{\iid}{i.i.d.\xspace}
\newcommand{\KL}[2]{\mD_{\mathtt{KL}}(#1 \,\|\, #2)} 

\newcommand{\B}{B_{\epsilon}}
\newcommand{\term}[1]{\ensuremath{\mathtt{#1}}\xspace}
\newcommand{\SemiBwK}{\term{SemiBwK}}
\newcommand{\BwK}{\term{BwK}}
\newcommand{\AD}{\term{linCBwK}}

\newcommand{\MB}{\term{OMM}}

\newcommand{\ChernoffC}{c_{\term{conf}}}

\newcommand{\semiBwKUCB}{\ensuremath{\algorithm}}
\newcommand{\pdBwK}{\term{pdBwK}}

\newcommand{\rew}{\term{rew}}
\newcommand{\regret}{\term{Regret}}

\newcommand{\StopTime}{\ensuremath{\tau_{\mathtt{stop}}}}
\newcommand{\BigD}{\ensuremath{\mD_{\mathtt{M}}}}

\newcommand{\T}{\ensuremath{\tau}}
\renewcommand{\vec}[1]{\boldsymbol{#1}}
\newcommand{\LPOptAlg}[1][]{\OPT_{\mathtt{ALG},\,#1}}
\newcommand{\LPOptAtoms}{\OPT_{\mathtt{atoms}}}
\newcommand{\LPOptSubsets}{\OPT_{\mathtt{subsets}}}
\newcommand{\LPOptDistr}{\OPT_{\mathtt{BwK}}}

\makeatletter
\def\blfootnote{\xdef\@thefnmark{}\@footnotetext}
\makeatother

\title{Combinatorial Semi-Bandits with Knapsacks\thanks{Extended abstract appears in the 21st International Conference on Artificial Intelligence and Statistics (\emph{AIStats 2018}).}}
\author{Karthik A. Sankararaman\thanks{\textbf{Email: }\texttt{kabinav@cs.umd.edu}Supported in part by NSF Awards CNS 1010789 and CCF 1422569.}}
\affil{University of Maryland, College Park}
\author{Aleksandrs Slivkins\thanks{\textbf{Email: } \texttt{slivkins@microsoft.com}}}
\affil{Microsoft Research NYC}

\date{}

\begin{document}
\maketitle

\begin{abstract}
We unify two prominent lines of work on multi-armed bandits: \emph{bandits with knapsacks} and \emph{combinatorial semi-bandits}. The former concerns limited ``resources" consumed by the algorithm, \eg limited supply in dynamic pricing. The latter allows a huge number of actions but assumes combinatorial structure and additional feedback to make the problem tractable. We define a common generalization, support it with several motivating examples, and design an algorithm for it. Our regret bounds are comparable with those for \BwK and combinatorial semi-bandits.
\end{abstract}

\section{Introduction}
Multi-armed bandits (\emph{MAB}) is an elegant model for studying the tradeoff between acquisition and usage of information, a.k.a.\ \emph{explore-exploit tradeoff} \citep{Robbins1952, Thompson-1933}. In each round an algorithm sequentially chooses from a fixed set of alternatives (sometimes known as \emph{actions} or \emph{arms}), and receives reward for the chosen action. Crucially, the algorithm does not have enough information to answer all ``counterfactual" questions about what would have happened if a different action was chosen in this round. MAB problems have been studied steadily since 1930-ies, with a huge surge of interest in the last decade.


This paper combines two lines of work related to bandits: on \emph{bandits with knapsacks} (\BwK)~\citep{BwK-focs13} and on \emph{combinatorial semi-bandits}~\citep{Gyorgy-jmlr07}. \BwK concern scenarios with limited ``resources" consumed by the algorithm, \eg limited inventory in a dynamic pricing problem. In combinatorial semi-bandits, actions correspond to subsets of some ``ground set", rewards are additive across the elements of this ground set (\emph{atoms}), and the reward for each chosen atom is revealed (\emph{semi-bandit feedback}). A paradigmatic example is an online routing problem, where atoms are edges in a graph, and actions are paths. Both lines of work have received much recent attention and are supported by numerous examples.

\xhdr{Our contributions.}
We define a common generalization of combinatorial semi-bandits and \BwK, termed \emph{Combinatorial Semi-Bandits with Knapsacks} (\SemiBwK). Following all prior work on \BwK, we focus on an \iid environment: in each round, the ``outcome" is drawn independently from a fixed distribution over the possible outcomes. Here the ``outcome" of a round is the matrix of reward and resource consumption for all atoms.%
\footnote{Our model allows arbitrary correlations within a given round, both across rewards and consumption for the same atom and across multiple atoms. Such correlations are essential in applications such as dynamic pricing and dynamic assortment. \Eg customers' valuations can be correlated across products, and algorithm earns only if it sells; see Section~\ref{sec:applications} for details.}
We design an algorithm for \SemiBwK, achieving regret rates that are comparable with those for \BwK and combinatorial semi-bandits.

Specifics are as follows. As usual, we assume ``bounded outcomes": for each atom and each round, rewards and consumption of each resource is non-negative and at most $1$. Regret is relative to the expected total reward of the best all-knowing policy, denoted $\OPT$. For \BwK problems, this is known to be a much stronger benchmark than the traditional best-fixed-arm benchmark. We upper-bound the regret in terms of the relevant parameters: time horizon $T$, (smallest) budget $B$, number of atoms $n$, and $\OPT$ itself (which may be as large as $nT$). We obtain
\begin{align}\label{eq:main-regret}
\regret \leq \tilde{O}(\sqrt{n})( \OPT /\sqrt{B} + \sqrt{T+\OPT}).
\end{align}
The ``shape" of the regret bound is consistent with prior work: the $\OPT /\sqrt{B}$ additive term appears in the optimal regret bound for \BwK, and the $\sqrt{T}$ and $\sqrt{\OPT}$ additive terms are very common in regret bounds for MAB.  The per-round running time is polynomial in $n$, and near-linear in $n$ for some important special cases.

\asedit{Our regret bound is optimal up to $\polylog$ factors for paradigmatic special cases. \BwK is a special case when actions are atoms. For $\OPT>\Omega(T)$, the regret bound becomes
    $\tilde{O}( T \sqrt{n/B} + \sqrt{nT})$,
where $n$ is the number of actions, which coincides with the
lower bound from \citep{BwK-focs13}. Combinatorial semi-bandits is a special case with $B=nT$. If all feasible subsets contain at most $k$ atoms, we have $\OPT\leq kT$, and the regret bound becomes
    $\tilde{O}(\sqrt{knT})$. This coincides with the  $\Omega(\sqrt{knT})$ lower bound from~\citep{matroidBandit}.}

Our main result assumes that the action set, \ie the family of feasible subsets of atoms, is described by a \emph{matroid constraint}.%
\footnote{Matroid is a standard notion in combinatorial optimization which abstracts and generalizes linear independence.}
This is a rather general scenario which includes many paradigmatic special cases of combinatorial semi-bandits such as  cardinality constraints, partition matroid constraints, and spanning tree constraints. We also assume that
    $B> \tilde{\Omega}(n+\sqrt{nT})$.


\asedit{Our model captures several application scenarios, incl. dynamic pricing, dynamic assortment, repeated auctions, and repeated bidding. We work out these applications, and explain how our regret bounds improve over prior work.}

\OMIT{\kaedit{
\xhdr{Notable applications.} Our general model captures multiple important applications (from economics, crowdsourcing, auctions, etc.) and generalizes those captured by \BwK and the related models. We now briefly introduce one of them here and defer the rest to Section~\ref{sec:applications}. The application we consider is the dynamic pricing problem with multiple goods \cite{BZ09}. We have $d$ different products on sale with $B$ copies of each product. In each time-step an agent arrives with different valuations for each of the products (these valuations may be correlated). The algorithm simultaneously sets prices for the products (from a discrete set of price points). The agent then buys all items where the valuation is at least the price for that product. The profit the algorithm receives for this time-step is the price of all items bought by this agent. The goal is to maximize the total profits. This naturally fits into our framework, where we have an atom for every item and a price point. in each step, the algorithm is allowed to choose at most one atom per product (which would correspond to the price for that product) and the goal is to select the ``best subset''. This can naturally be captured by a combinatorial constraint such as a partition matroid \citep{papadimitriou1982combinatorial}. The goal is to maximize the total reward subject to the budget constraints.
}}

\xhdr{Challenges and techniques.} \BwK problems are challenging compared to traditional MAB problems with \iid rewards because it no longer suffices to look for the best action and/or optimize expected per-round rewards; instead, one essentially needs to look for a \emph{distribution} over actions with optimal expected \emph{total} reward across all rounds. Generic challenges in combinatorial semi-bandits concern handling exponentially many actions (both in terms of regret and in terms of the running time), and taking advantage of the additional feedback. And in \SemiBwK, one needs to deal with distributions over subsets of atoms, rather than ``just" with distributions over actions.

\asedit{Our algorithm connects a technique from \BwK and a randomized rounding technique from combinatorial optimization. (With \emph{five} existing \BwK algorithms and a wealth of approaches for combinatorial optimization, choosing the techniques is a part of the challenge.)}

We build on a \BwK algorithm from \citet{AgrawalDevanur-ec14}, which combines linear relaxations and a well-known "optimism-under-uncertainty" paradigm. A generalization of this algorithm to \SemiBwK results in a fractional solution $\vec{x}$, a vector over atoms. \asedit{Randomized rounding converts $\vec{x}$ into a distribution over feasible subsets of atoms that equals $\vec{x}$ in expectation. It is crucial (and challenging) to ensure that this distribution contains enough randomness so as to admit} concentration bounds not only across rounds, but also across atoms. Our analysis "opens up" a fairly technical proof from prior work and intertwines it with a new argument based on negative correlation.

We present our algorithm and analysis so as to "plug in" any suitable randomized rounding technique. This makes our presentation more lucid, and also leads to faster running times for important special cases.

\xhdr{Solving \SemiBwK using prior work.} Solving \SemiBwK using an algorithm for \BwK would result in a regret bound like \eqref{eq:main-regret} with $n$ replaced with the number of actions. The latter could be on the order of $n^k$ if each action can consist of at most $k$ atoms, or perhaps even exponential in $n$.

\SemiBwK can be solved as a special case of a much more general  \emph{linear-contextual} extension of \BwK from \citet{AgrawalDevanur-ec14,CBwK-nips16}. In their model, an algorithm takes advantage of the combinatorial structure of actions, yet it ignores the additional feedback from the atoms. Their regret bounds have a worse dependence on the parameters, and apply for a much more limited range of parameters. Further, their per-round running time is linear in the number of actions, which is often prohibitively large.

To compare the regret bounds, let us focus on instances of \SemiBwK in which at most one unit of each resource is consumed in each round. (This is the case in all our motivating applications, except repeated bidding.) Then \citet{AgrawalDevanur-ec14,CBwK-nips16} assume $B>\sqrt{n}\, T^{3/4}$, and achieve regret
    $\tilde{O}(n \sqrt{T}\tfrac{\OPT}{B}+ n^2 \sqrt{T})$.
\footnote{\citet{AgrawalDevanur-ec14,CBwK-nips16} state regret bound with term $+n\sqrt{T}$ rather than $+n^2 \sqrt{T}$, but they assume that per-round rewards lie in $[0,1]$. Since per-round rewards can be as large as $n$ in our setting, we need to scale down all rewards by a factor of $n$, apply their regret bound, and then scale back, which results in the regret bound with $+n^2 \sqrt{T}$. When per-round consumption can be as large as $n$, regret bound from \citet{AgrawalDevanur-ec14,CBwK-nips16} becomes
    $\tilde{O}(n^2\OPT \sqrt{T}/B+  n^2\sqrt{T})$
due to rescaling.}
It is easy to see that we improve upon the range and upon both summands. In particular, we improve both summands  by the factor of $n\sqrt{n}$ in a lucid special case when $B>\Omega(T)$ and $\OPT<O(T)$.%
\footnote{\asedit{In prior work on combinatorial bandits (without constraints), semi-bandit feedback improves regret bound by a factor of $\sqrt{n}$, see the discussion in \cite{tightRegret}.}}

We run simulations to compare our algorithm against prior work on \BwK and combinatorial semi-bandits.

\OMIT{\kaedit{Note similar improvements have been observed in the combinatorial semi-bandits without knapsack settings (see the discussion in \cite{tightRegret}) where assuming semi-bandit feedback helps in removing a factor of $\sqrt{K}$ (where $K$ is the maximum size of a feasible subset) from the regret. Additionally, the initial exploration time needed for prior work is $\binom{n}{k}$ due to the increased number of ``arms'' in the reduction.}}

\xhdr{Related work.}
Multi-armed bandits have been studied since \cite{Thompson-1933} in Operations Research, Economics, and several branches of Computer Science, see \citep{Gittins-book11,Bubeck-survey12} for background. Among broad directions in MAB, most relevant is MAB with \iid rewards, starting from \citep{Lai-Robbins-85,bandits-ucb1}.

Bandits with Knapsacks (\BwK) were first introduced by \citet{BwK-focs13} as a common generalization of several models from prior work and many other motivating examples. Subsequent papers extended \BwK to  ``smoother" resource constraints and introduced several new algorithms \citep{AgrawalDevanur-ec14}, and generalized \BwK to contextual bandits \citep{cBwK-colt14,CBwK-colt16,CBwK-nips16}.  All prior work on \BwK and special cases thereof assumed i.i.d. outcomes.

Special cases of \BwK include dynamic pricing with limited supply \citep{DynPricing-ec12,BZ09,BesbesZeevi-or12,Wang-OR14}, dynamic procurement on a budget \citep{DynProcurement-ec12,Krause-www13,Crowdsourcing-PositionPaper13}, dynamic ad allocation with advertiser budgets \citep{AdsWithBudgets-arxiv13}, and bandits with a single deterministic resource \citep{GuhaM-icalp09,GuptaKMR-focs11,TranThanh-aaai10,TranThanh-aaai12}.
Some special cases admit instance-dependent logarithmic regret bounds \citep{XiaIJCAI,xia2016budgeted,combes2015bandits,AdsWithBudgets-arxiv13} when there is only one bounded resource and unbounded time, or when resource constraints do not bind across arms.


Combinatorial semi-bandits were studied by \citet{Gyorgy-jmlr07}, in the adversarial setting. In the i.i.d. setting, in a series of works by \citep{anantharam1987asymptotically,gai2010learning,gai2012,chen13a,tightRegret,combes2015combinatorial}, an optimal algorithm was achieved. This result was then extended to atoms with linear rewards by \cite{largeScale}. \cite{matroidBandit} obtained improved results for the special case when action set is described by a matroid. Some other works studied a closely related ``cascade model", where the ordering of atoms matters
\citep{cascadingicml15,DCMBandits,cascadingUAI16}. Contextual semi-bandits have been studied in \citep{largeScale,semi-CB-nips16}.

Randomized rounding schemes (RRS) come from the literature on approximation algorithms in combinatorial optimization (see \cite{williamson2011design,papadimitriou1982combinatorial} for background).
RRS were introduced in \cite{raghavan1987randomized}.
Subsequent work \citep{gandhi2006dependent,asadpour2010log,chekuri2010dependent,chekuri2011multi} developed RRS which correlate the rounded random variables so as to guarantee  sharp concentration bounds.

\xhdr{Discussion.}
\asedit{The basic model of multi-armed bandits can be extended in many distinct directions: what auxiliary information, if any, is revealed to the algorithm before it needs to make a decision, which feedback is revealed afterwards, which ``process" are the rewards coming from, do they have some known structure that can be leveraged, are there global constraints on the algorithm, etc. While many real-life scenarios combine several directions, most existing work proceeds along only one or two. We believe it is important (and often quite challenging) to unify these lines of work. For example, an important recent result of \citet{Syrgkanis-AdvCB-icml16,Rakhlin-AdvCB-icml16} combined ``contextual" and ``adversarial" bandits.}

\OMIT{ 
\xhdr{Open questions.}
One question concerns an extension to linear rewards, in line with a large literature on linear bandits. Here, each atom is characterized by a known low-dimensional feature vector, and its reward and resource consumption is a linear function of the features. As in \citep{largeScale}, the goal is to alleviate the dependence on number of atoms by replacing it with the dependence on number of features.
} 

\OMIT{If the feature vectors change from one round to another, this corresponds to contextual bandits, or rather a common generalization of contextual bandits, semi-bandits, and \BwK. It is worth noting that the existing work
(\citep{semi-CB-nips16,cBwK-colt14,CBwK-colt16,CBwK-nips16} and this paper)  has combined \emph{any two} of the three. }

\section{Our model and preliminaries}
\label{sec:prelims}

Our model, called \emph{Semi-Bandits with Knapsacks} (\SemiBwK) is a generalization of multi-armed bandits (henceforth, \emph{MAB}) with \iid rewards. As such, in each round $t=1 \LDOTS T$, an algorithm chooses an action $S_t$ from a fixed set of actions $\mF$, and receives a reward $\mu_t(S_t)$ for this action which is drawn independently from a fixed distribution that depends only on the chosen action. The number of rounds $T$, a.k.a. the \emph{time horizon}, is known.

There are $d$ resources being consumed by the algorithm. The algorithm starts out with budget $B_j\geq 0$ of each resource $j$. All budgets  are known to the algorithm. If in round $t$ action $S\in \mF$ is chosen, the outcome of this round is not only the reward $\mu_t(S)$ but the consumption $C_t(S,j)$ of each resource $j\in[d]$. We refer to
    $\vec{C}_t(S) = \left( C_t(S,j):\; j\in[d] \right)$
as the \emph{consumption vector}.%
\footnote{We use bold font to indicate vectors and matrices.}
Following prior work on \BwK, we assume that all budgets are the same: $B_j = B$ for all resources $j$.%
\footnote{This is w.l.o.g. because we can divide all consumption of each resource $j$ by $B_j/\min_{j' \in [d]} B_{j'}$.
Effectively, $B$ is the smallest budget in the original problem instance.}
Algorithm stops as soon as any one of the resources goes strictly below 0. The round in which this happens is called the stopping time and denoted \StopTime.  The reward collected in this last round does not count; so the total reward of the algorithm is
    $ \textstyle \rew = \sum_{t<\StopTime} \; \mu_t(S_t) $.

Actions correspond to subsets of a finite ground set $\mA$, with $n=|\mA|$; we refer to elements of $\mA$ as \emph{atoms}. Thus, the set $\mF$ of actions corresponds to the family of ``feasible subsets" of $\mA$. The rewards and resource consumption is additive over the atoms: for each round $t$ and each atom $a$ there is a reward $\mu_t(a)\in [0,1]$ and consumption vector $\vec{C}_t(a)\in [0,1]^d$ such that for each action $S\subset \mF$ it holds that
    $\mu_t(S) = \sum_{a\in S} \mu_t(a)$
and
    $\vec{C}_t(S) = \sum_{a\in S} \vec{C}_t(a)$.

We assume the \iid property across rounds, but allow arbitrary correlations within the same round. Formally, for a given round $t$ we consider the $n\times (d+1)$ ``outcome matrix"
    $(\mu_t(a), \vec{C}_t(a): a\in \mA)$,
which specifies rewards and resource consumption for all resources and all atoms. We assume that the outcome matrix is chosen independently from a fixed distribution $\BigD$ over such matrices. The distribution $\BigD$ is not revealed to the algorithm. The mean rewards and mean consumption is denoted
    $\mu(a) := \E[\mu_t(a)]$ and $\vec{C}(a) := \E[\vec{C}_t(a)] $.
We extend the notation to actions, \ie to subsets of atoms:
    $\mu(S) := \sum_{a\in S} \mu(a)$
and $\vec{C}(S) := \sum_{a\in S} \vec{C}(a)$.

An instance of \SemiBwK consists of the action set $\mF\subset 2^{[n]}$, the budgets $\vec{B} = (B_j:\; j\in [d])$,  and the distribution $\BigD$. The $\mF$ and $\vec{B}$ are known to the algorithm, and $\BigD$ is not. As explained in the introduction, \SemiBwK subsumes  \emph{Bandits with Knapsacks} (\BwK) and semi-bandits. \BwK is the special case when $\mF$ consists of singletons, and semi-bandits is the special case when all budgets are equal to $B_j = nT$ (so that the resource consumption is irrelevant).

Following the prior work on \BwK, we compete against the ``optimal all-knowing algorithm": an algorithm that optimizes the expected total reward for a given problem instance; its expected total reward is denoted by $\OPT$. As observed in \cite{BwK-focs13}, $\OPT$ can be much larger (\eg factor of 2 larger) than the expected cumulative reward of the best action, for a variety of important special cases of \BwK. Our goal is to minimize \emph{regret}, defined as $\OPT$ minus algorithm's total reward.

\OMIT{We strive to upper-bound regret in terms of the key parameters: the time horizon $T$, the optimal value $\OPT$, the number of atoms $n$ and the smallest budget $\min_{j\in [d]} B_j$. We may also factor in the smallest size of a feasible subset: $k:= \min_{S\in\mF} |S|$, particularly when $k \ll n$.}


\xhdr{Combinatorial constraints.}
Action set $\mF$ is given by a \emph{combinatorial constraint}, \ie a family of subsets. Treating subsets of atoms as $n$-dimensional binary vectors, $\mF$ corresponds to a finite set of points in $\R^n$. We assume that the convex hull of $\mF$ forms a polytope in $\R^n$. In other words, there exists a set of linear constraints over $\R^n$ whose set of feasible \emph{integral} solutions is \mF. We call such $\mF$ \emph{linearizable}; the convex hull is called the polytope \emph{induced} by $\mF$.

Our main result is for \emph{matroid constraints}, a family of  linearizable combinatorial constraints which subsumes several important special cases such as cardinality constraints, partition matroid constraints, spanning tree constraints and transversal constraints. \asedit{Formally, $\mF$ is a matroid if it contains the empty set, and satisfies two properties: (i) if $\mF$ contains a subset $S$, then it also contains every subset of $S$, and (ii) for any two subsets $S,S'\in \mF$ with $|S|>|S'|$ it holds that $S'\cup \{a\} \in \mF$ for each atom $a\in S\setminus S'$. See Appendix~\ref{appx:matroid} for more background and examples.}

\OMIT{
The convex hull of $\mF$, a.k.a. the \emph{matroid polytope}, can be represented via the following linear system:
\begin{equation*}
\begin{array}{ll@{}ll}
\tag{LP-Matroid}
\label{pt:matroid}
x(S) \leq \text{rank}(S) & \forall S \subseteq E \\
x_e \in [0,1]^E & \forall e \in E.
\end{array}
\end{equation*}

\noindent Here
    $x(S) := \sum_{e \in S} x_e$,
and
    $\text{rank}(S) = \max \{ |Y| : Y \subseteq S, Y \in \mF\}$
is the ``rank function" for $\mF$.

$\mF$ is indeed the set of all feasible integral solutions of the above system. This is a standard fact in combinatorial optimization, \eg see Theorem 40.2 and its corollaries in \citet{schrijver2002combinatorial}.
}

We incorporate prior work on randomized rounding for linear programs. Consider a linearizable action set $\mF$ with induced polytope $P\subset [0,1]^n$. The \emph{randomized rounding scheme} (henceforth, $\RRS$) for $\mF$ is an algorithm which inputs a feasible fractional solution $\vec{x}\in P$  and the linear equations describing $P$, and produces a random vector $\vec{Y}$ over $\mF$. We consider $\RRS$'s such that
    $\E[\vec{Y}] = \vec{x}$
and $\vec{Y}$ is negatively correlated (see below for definition); we call such $\RRS$'s \emph{negatively correlated}. Several such $\RRS$ are known: \eg
for cardinality constraints and bipartite matching \citep{gandhi2006dependent},
for spanning trees \citep{asadpour2010log}, and for matroids \citep{chekuri2010dependent}.




\xhdr{Negative correlation.}
Let $\mX = (X_1, X_2, \ldots, X_m)$ denote a family of random variables which take values in $[0,1]$. Let $X := \frac{1}{m} \sum_{i=1}^m X_i$ be the average, and $\mu := \E[X]$.

Family $\mX$ is called \emph{negatively correlated} if
\begin{align}
	\E\left[ \prod_{i \in S} X_i \right]
    \leq \prod_{i \in S}\E[X_i] \quad \forall S \subseteq [m]
    \label{eq:neg-cor-defn}\\
	\E\left[ \prod_{i \in S} (1-X_i) \right]
    \leq \prod_{i \in S}\E[1-X_i] \; \forall S \subseteq [m]
    \label{eq:neg-cor-defn-2}
\end{align}


\asedit{Independent random variables satisfy both properties with equality. For intuition: if $X_1,X_2$ are Bernoulli and \eqref{eq:neg-cor-defn} is strict, then $X_1$ is more likely to be $0$ if $X_2=1$.

Negative correlation is a generalization of independence that allows for similar \emph{concentration bounds}, \ie high-probability upper bounds on $|X-\mu|$. However, our analysis does not invoke them directly. Instead, we use a concentration bound given a closely related property:}
\begin{align}\label{eq:neg-corr-half}
\E\left[ \prod_{i \in S}\;X_i \right] \leq (\tfrac12)^{|S|} \quad \forall S \subseteq [m].
\end{align}

\begin{theorem}\label{thm:prelim:additive}
If \eqref{eq:neg-corr-half}, then for some absolute constant $c$,
\begin{align}\label{eq:thm:prelim:additive}
    \Pr[X \geq \tfrac12 + \eta] \leq c \cdot e^{-2m\eta^2}
    \qquad (\forall \eta>0)
\end{align}
\end{theorem}

This theorem easily follows from \citep{IKChernoff}, see Appendix~\ref{app:prob}.

\xhdr{Confidence radius.}
We bound deviations $|X-\mu|$ in a way that gets sharper when $\mu$ is small, without knowing $\mu$ in advance. (We use the notation $\mX,X,\mu$ as above.) To this end, we use the notion of \emph{confidence radius} from \citep{LipschitzMAB-merged-arxiv,DynPricing-ec12,BwK-focs13,agrawal2014bandits}\footnote{For instance Theorem 2.1 in \citep{BwK-full}}:
\begin{align}\label{eq:conf-rad}
\rad_{\alpha}(x, m) = \sqrt{\alpha x/m} + \alpha/m.
\end{align}
\noindent If random variables $\mX$ are independent, then event
\begin{align}
\label{eq:conf-rad-prop}
|X - \mu| < \rad_{\alpha}(X, m) < 3 \rad_{\alpha}(\mu, m)
\end{align}
happens with probability at least $1- O( e^{-\Omega(\alpha)} )$, for any given $ \alpha>0$.
We use this notion to define upper/lower confidence bounds on the mean rewards and mean resource consumption. Fix round $t$, atom $a$, and resource $j$. Let $\hat{\mu}_t(a)$ and $\hat{C}_t(a, j)$ denote the empirical average of the rewards and resource-$j$ consumption, resp., between rounds $1$ and $t-1$. Let $N_t(a)$ be the number of times atom $a$ has been chosen in these rounds (\ie included in the chosen actions). The confidence bounds are defined as
\begin{align}
C_t^{\pm}(a, j)
    &= \term{proj}(\;\hat{C}(a, j) \pm \rad_\alpha(\hat{C}(a, j), N_t(a))\;)
    \nonumber\\
\mu_t^{\pm}(a)
    &= \term{proj}\left(\; \hat{\mu}(a) \pm \rad_\alpha(\hat{\mu}(a), N_t(a)) \;\right)\label{eq:confidencebounds}
\end{align}
where
    $\term{proj}(x) := \argmin_{y\in[0,1]} |y-x| $ denotes the projection into $[0,1]$.
We always use the same parameter $\alpha = \ChernoffC \,\log(ndT)$, for an appropriately chosen absolute constant $\ChernoffC$. We suppress $\alpha$ and $\ChernoffC$
from the notation. We use a vector notation
    $\vec{\mu}_t^{\pm}$ and $\vec{C}_t^{\pm}(j)$
to denote the corresponding $n$-dimensional vectors over all atoms $a$.

 By \eqref{eq:conf-rad-prop},
with probability $1-O(e^{-\Omega(\alpha)})$ the following hold.
\begin{align*}
\mu(a) &\in [\mu_t^-(a),\;\mu_t^+(a)]\\
C(a, j)&\in [C^-(a, j),\;C(a, j)^+]
\end{align*}

\OMIT{ 

We use a version of the Chernoff-Hoeffding bounds for negatively correlated random variables \citep{PanconesiSrinivasan}. More precisely, we use a version from \citep{IKChernoff} which focuses on a weaker condition:
\begin{align}\label{eq:neg-cor-weaker}
\textstyle \exists \nu_1 \LDOTS \nu_m \in [0,1] \qquad
\E\left[ \prod_{i \in S} X_i \right]
    \leq \prod_{i\in S} \nu_i \quad \forall S \subseteq [m].
\end{align}
Note that negatively correlated random variables satisfy \eqref{eq:neg-cor-weaker} with $\nu_i=\mu_i$ for all $i$.

\begin{theorem}
\label{thm:prelim:negChernoff}
Suppose family $\mX$ satisfies \eqref{eq:neg-cor-weaker} and let
    $\nu = \tfrac{1}{m} \sum_{i\in [m]} \nu_i$.
Then:
\begin{align}\label{eq:thm:prelim:negChernoff}
\exists c>0\qquad \forall \gamma\in[\nu, 1]\qquad
\Pr[X\geq \gamma] \leq c\cdot \exp(-n\; \KL{\gamma}{\nu}),
\end{align}
where $\KL{\cdot}{\cdot}$ denotes the KL-divergence. Whenever
    $\nu\leq \mu$, we have
\begin{OneLiners}
\item[(a)]
  $\Pr[|X - \mu| > \eps \mu] \leq c\; \exp(-\mu m \eps^2 / 3)$
  for any $\eps>0$.

\item[(b)]		$\Pr[X > a] \leq 2^{-am}$ for any $a > 6\mu$.
\end{OneLiners}
	\end{theorem}

Parts (a,b) of the theorem are the standard properties derived in Chernoff Bounds, except here they hold under a much more general condition: \eqref{eq:neg-cor-weaker} with $\nu\leq\mu$. To derive part (a), we take \eqref{eq:thm:prelim:negChernoff} with  $\gamma = \mu(1 + \eps)$ for upper tail and $\gamma = \mu(1-\eps)$ for the lower tail; part (b) follows from (a).

} 
\section{Main algorithm}


Let us define our main algorithm, called $\algorithm$. The algorithm builds on an arbitrary $\RRS$ for the action set $\mF$. It is parameterized by this $\RRS$, the polytope $\mP$ induced by $\mF$ (represented as a collection of linear constraints), and a number $\eps>0$. In each round $t$, it recomputes the upper/lower confidence bounds, as defined in \eqref{eq:confidencebounds}, and solves the following linear program:
	\begin{equation*}
	\begin{array}{ll@{}ll}
	\label{lp:template}
	\tag{$\LP_{\ALG}$}
	\text{maximize}  & \vec{\mu_t^+}\cdot\vec{x}& &\\
	\text{subject to}& \vec{C}_t^-(j)\cdot\vec{x} \leq \frac{B(1-\epsilon)}{T},
        &&j\in[d]\\
	& \vec{x} \in \mathcal{P} &  &
	\end{array}
	\end{equation*}
This linear program defines a linear relaxation of the original problem which is ``optimistic" in the sense that it uses upper confidence bounds for rewards and lower confidence bounds for consumption. The linear relaxation is also ``conservative" in the sense that it rescales the budget by $1-\eps$. Essentially, this is to ensure that the algorithm does not run out of budget with high probability. Parameter $\eps$ will be fixed throughout. For ease of notation, we will denote $\B := (1-\eps)B$ henceforth. The LP solution $\vec{x}$ can be seen as a probability vector over the atoms. Finally, the algorithm uses the $\RRS$ to convert the LP solution into a feasible action. The pseudocode is given as Algorithm~\ref{alg:ucbtemplate}.

\begin{algorithm2e}[!h]
\caption{$\algorithm$}
\label{alg:ucbtemplate}
\DontPrintSemicolon
\SetKwInOut{Input}{input}

\Input{an $\RRS$ for action set $\mF$, induced polytope $\mathcal{P}$ (as  a set of linear constraints), $\epsilon>0$.}

		\For{$t = 1, 2 \LDOTS T$}{
			\begin{enumerate}
			\item
			\textbf{Recompute Confidence Bounds} as in  \eqref{eq:confidencebounds} \;
\item \textbf{Obtain fractional solution $\vec{x}_t\in [0,1]^n$} by solving the linear program \ref{lp:template}.

\item \textbf{Obtain a feasible action $S_t\in\mF$} by invoking the $\RRS$ on vector $\vec{x}_t$.

\item \textbf{Semi-bandit Feedback}: observe the rewards/consumption for all atoms $a\in S_t$.
\end{enumerate}
		}
\end{algorithm2e}

If action set $\mF$ is described by a matroid constraint, we can use the negatively correlated $\RRS$ from \cite{chekuri2010dependent}. In particular, we obtain a complete algorithm for several combinatorial constraints commonly used in the literature on semi-bandits, such as partition matroid constraints, spanning trees. More background on matroid constraints can be found in the Appendix \ref{appx:matroid}.

\begin{theorem}
\label{mainresult}
Consider the \SemiBwK problem with a linearizable action set $\mF$ that admits a negatively correlated $\RRS$. Then algorithm $\algorithm$ with this $\RRS$ achieves expected regret bound at most
\begin{align} \label{eq:mainresult}
O(\log(ndT)) \;\sqrt{n}
    \left( \OPT /\sqrt{B} + \sqrt{T+\OPT}\right).
\end{align}
Here $T$ is the time horizon, $n$ is the number of atoms, and $B$ is the budget. We require
    $B > 3(\alpha n + \sqrt{\alpha nT})$,
where $\alpha = \Theta(\log(ndT))$ is the parameter in confidence radius. Parameter $\eps$ in the algorithm is set to
    $\sqrt{\frac{\alpha n}{B}} + \frac{\alpha n}{B} + \frac{\sqrt{\alpha nT}}{B}$.
\end{theorem}

\begin{corollary}\label{maincor}
Consider the setting in Theorem~\ref{mainresult} and assume that the action set $\mF$ is defined by a matroid on the set of atoms. Then, using the negatively correlated $\RRS$ from \citep{chekuri2010dependent}, we obtain regret bound \eqref{eq:mainresult}.
\end{corollary}

\xhdr{Running time of the algorithm.}
The algorithm does two computationally intensive steps in each round: solves the linear program \eqref{lp:template} and runs the $\RRS$. For matroid constraints, the $\RRS$ from \cite{chekuri2010dependent} has $O(n^2)$ running time. Hence, in the general case the computational bottleneck is solving the LP, which has $n$ variables and $O(2^n)$ constraints. Matroids are known to admit a polynomial-time seperation oracle \citep[\eg see][]{schrijver2002combinatorial}. It follows that the entire set of constraints in \ref{lp:template} admits a polynomial-time separation oracle, and therefore we can use the Ellipsoid algorithm to solve \ref{lp:template} in polynomial time.  For some classes of matroid constraints the LP is much smaller: \eg for cardinality constraints (just $d+1$ constraints) and for traversal matroids on bipartite graphs (just $2n+d$ constraints). Then near-linear-time algorithms can be used.


Our algorithm works under any negatively correlated RRS. We can use this flexibility to improve the per-round running time for some special cases. (Making decisions extremely fast is often critical in practical applications of bandits \citep[\eg see][]{MWT-WhitePaper-2016}.) We obtain near-linear per-round running times for cardinality constraints and partition matroid constraints. Indeed, \ref{lp:template} can be solved in near-linear time, as mentioned above, and we can use a negatively correlated RRS from \citep{gandhi2006dependent} which runs in linear time. These classes of matroid constraints are important in our applications (see Section~\ref{sec:applications}).


\OMIT{ 
\subsection{Extension to linearizable action sets}

We extend our analysis to any linearizable action set, assuming each resource is consumed by at most one atom. (\Eg this is the case for the ``dynamic assortment" problem, see Section~\ref{sec:applications}.) We use a very simple $\RRS$: given a fractional solution $\vec{x}$ which lies in $\mP$, the polytope induced by the action set $\mF$, we represent $\vec{x}$ as a distribution $\vec{Y}$ over the vertices of $\mP$, and output $\vec{Y}$. This is a valid $\RRS$ because vertices of $\mP$ lie in $\mF$. However, while we get
    $\E[\vec{Y}]=\vec{x}$,
we cannot guarantee negative correlation or any other similarly useful property.

\begin{theorem}\label{thm:naive}
Consider the \SemiBwK problem with a linearizable action set. Assume each resource can be consumed by at most one atom. Use the same notation and same parameter $\eps$ as in Theorem \ref{mainresult}. Then algorithm $\algorithm$ with the simple $\RRS$ described above achieves expected regret at most
\begin{align}
 \label{eq:thm:naive}
O(\log(ndT/\delta)) \;
    ( \OPT \sqrt{n/B} + n\; \sqrt{\OPT} + n ).
\end{align}
The result holds if $B>\alpha n$, where $\alpha = \log(ndT/\delta)$.
Parameter $\eps$ in the algorithm is set to
    $\eps = \sqrt{\frac{\alpha n}{B}} + \frac{\alpha n}{B}$.
\end{theorem}

Compared to the main result, Theorem~\ref{thm:naive} makes a significant assumption that resources correspond to atoms. On the other hand, it relaxes the assumption on budget $B$. The regret bound is incomparable with \refeq{eq:mainresult}: for example, it is worse when $T/n \ll \OPT <B$, and better when
    $\OPT \ll T/n < nB$.
The theorem is proved similarly to Theorem~\ref{mainresult}; the main modification is a simpler-to-prove but less efficient version of \refeq{eq:concentration-rewards-body} below.

In particular, this extension helps in the application to dynamic assortment with mutually exclusive products (see Section~\ref{sec:moreApplications}).

} 

\def\mainresult {\ref{mainresult}}
\section{Proof of Theorem \mainresult}
\label{sec:proof}

\xhdr{Proof overview.}
First, we argue that \ref{lp:template} provides a good benchmark that we can use instead of $\OPT$. Specifically, at any given round, the optimal value for \ref{lp:template} in each round is at least
$\frac{1}{T} (1-\eps) \OPT$
with high probability. We prove this by constructing  a series of LPs, starting with a generic linear relaxation for \BwK and ending with \ref{lp:template}, and showing that the optimal value does not decrease along the series.

Next we define an event that occur with high probability, henceforth called \emph{clean event}. This event concerns total rewards, and compares our algorithm against \ref{lp:template}:
\begin{equation}\label{eq:concentration-rewards-body}
        \textstyle |\sum_{t\in [T]}\; r_t - \sum_{t\in [T]}\; \vec{\mu_t^+}\cdot\vec{x_t}|
        \leq \textstyle O\left( \sqrt{\alpha n \sum_{t\in [T]}\; r_t} + \sqrt{\alpha n T} + \alpha n \right).
\end{equation}

We prove that it is indeed a high-probability event in three steps. First, we relate the algorithm's reward $\sum_t r_t$ to its expected reward $\sum_t \vec{\mu}\cdot S_t$, where we interpret the chosen action $S_t$, a subset of atoms, as a binary vector over the atoms. Then we relate
    $\sum_t \vec{\mu}\cdot S_t$
to
    $\sum_t \vec{\mu^+_t}\cdot S_t$,
replacing expected rewards with the upper confidence bounds.
Finally, we relate
    $\sum_t \vec{\mu^+_t}\cdot S_t$
to
    $\sum_t \vec{\mu^+_t}\cdot\vec{x_t}$,
replacing the output of the RRS with the corresponding expectations. Putting it together, we relate algorithm's reward to $\sum_t \vec{\mu^+_t}\cdot\vec{x_t}$, as needed.
It is essential to bound the deviations in the sharpest way possible; in particular, the naive $\tilde{O}(\sqrt{T})$ bounds are not good enough. To this end, we use several tools: the confidence radius from \eqref{eq:conf-rad}, the negative correlation property of the RRS, and another concentration bound from prior work.

A similar ``clean event'' (with a similar proof) concerns the total resource consumption of the algorithm. We condition on both clean events, and perform the rest of the analysis via a ``deterministic" argument not involving probabilities.  In particular, we use the second ``clean event" to guarantee that the algorithm never runs out of resources.

We use negative correlation via a rather delicate argument. We extend the concentration bound in Theorem~\ref{thm:prelim:additive} to a random process that evolves over time, and only assumes that property \eqref{eq:neg-corr-half} holds within each round conditional on the history. For a given round, we start with a negative correlation property of $S_t$ and construct another family of random variables that conditionally satisfies \eqref{eq:neg-corr-half}. The extended concentration bound is then applied to this family. The net result is a concentration bound for $\sum_t \vec{\mu^+_t}\cdot S_t$ \emph{as if} we had $n\times T$ independent random variables there.

The rest of the section contains the full proof.

\subsection{Linear programs}
\label{sec:LPs}

We argue that \ref{lp:template} provides a good benchmark that we can use instead of $\OPT$. Fix round $t$ and let $\LPOptAlg[t]$ denote the optimal value for \ref{lp:template} in this round. Then:

\begin{lemma}\label{clm:algopt}
$\LPOptAlg[t] \geq \frac{1}{T} (1-\eps) \OPT$
with probability at least $1- \delta$.
\end{lemma}

We will prove this by constructing  a series of LP's, starting with a generic linear relaxation for \BwK and ending with \ref{lp:template}. We show that along the series the optimal value does not decrease \whp.

The first LP, adapted from \cite{BwK-focs13}, has one decision variable for each action, and applies generically to any \BwK problem.
\begin{equation*}
\begin{array}{lll}
\label{LP:subsetsOriginal}
\tag{$\LP_{\BwK}$}
\text{maximize}  & \sum_{S \in\mF}\; \mu(S)\, x(S)&\\
\text{subject to}& \sum_{S\in\mF} \;  C(S, j)\, x(S) \leq B/T
&j=1 ,..., d\\
&0 \leq \sum_{S \in\mF}\; x(S) \leq 1.
\end{array}
\end{equation*}

Let $\LPOptDistr(B)$ denote the optimal value of this LP with a given budget $B$. Then:


\OMIT{ 
	Let $\Con$ denote the set of valid subsets of atoms that can be chosen, according to the semi-bandit constraint. Note that, we need to handle a hard constraint on the budgets $B$. Hence, we will use a slightly smaller value $\B = (1-\eps)B$. The following program hence, is the distribution assuming that the total available budget is $\B$.
	
	\begin{equation*}
	\begin{array}{lll}
	\label{LP:subsets}
	\tag{$\LP_{\subsets}$}
	\text{maximize}  & \displaystyle\sum\limits_{S \subseteq \mA: S \in \Con} \mu(S) x(S)& \\
	\text{subject to}& \displaystyle\sum\limits_{S \subseteq \mA: S \in \Con}   C(S, j)x(S) \leq \B/T,  &j=1 ,..., d\\
	& 0 \leq \sum\limits_{S \subseteq \mA: S \in \Con}x(S) \leq 1&\\
	\end{array}
	\end{equation*}
} 

\begin{claim}\label{clm:oneMinusEpsilon}
	$\LPOptDistr(\B) \geq (1-\eps)\LPOptDistr(B) \geq \frac{1}{T}(1-\eps)\, \OPT$.
\end{claim}

\begin{proof}
The second inequality in Claim~\ref{clm:oneMinusEpsilon} follows from \citep[Lemma 3.1 in ][]{BwK-focs13}. We will prove the first inequality as follows. Let $\vec{x^*}$ denote an optimal solution to \ref{LP:subsetsOriginal}(B). Consider $(1-\eps)x^*$; this is feasible to \ref{LP:subsetsOriginal}($\B$), since for every $S$, 
\[ (1-\eps)x^*(S) \leq 1 \quad\text{and}\quad 
    \displaystyle\sum\limits_{S \subseteq \mA: S \in \Con}   C(S, j)(1-\eps)x(S) \leq \B/T. \]
Hence, this is a feasible solution. Now, consider the objective function. Let $\vec{y}$ denote an optimal solution to \ref{LP:subsetsOriginal}($\B$). We have that
\[
\LPOptDistr(\B) = \displaystyle\sum\limits_{S \subseteq \mA: S \in \Con} \mu(S) y^*(S) \geq \displaystyle\sum\limits_{S \subseteq \mA: S \in \Con} \mu(S) (1-\eps)x^*(S) = (1-\eps)\LPOptDistr(B). \qedhere
\]
\end{proof}

Now consider a simpler LP where the decision variables correspond to atoms. As before, $\mP$ denotes the polytope induced by action set $\mF$.
\begin{equation*}
\begin{array}{llll}
\label{LP:arms}
\tag{$\LP_{\atoms}$}
\text{maximize}  & \vec{\mu}\cdot\vec{x}& &\\
\text{subject to}&  C^\dag\cdot\vec{x} \preccurlyeq \B/T&\vec{x} \in \mathcal{P}&\vec{x} \in [0,1]^n.
\end{array}
\end{equation*}
Here $C = (C(a,j):\; a\in A, j\in d)$ is the $n\times d$ matrix of expected consumption, and $C^\dag$ denotes its transpose. The notation $\preccurlyeq$ means that the inequality $\leq$ holds for for each coordinate.

Leting $\LPOptAtoms$ denote the optimal value for \ref{LP:arms}, we have:

\begin{claim}
	\label{clm:armsubsets}
	With probability at least $1-\delta$ we have, $\LPOptAlg[t] \geq \LPOptAtoms \geq \LPOptDistr(\B)$.
\end{claim}

\begin{proof}
			We will first prove the second inequality.
		
		Consider the optimal solution vector $\vec{x}$ to \ref{LP:subsetsOriginal}($\B$).
		Define $S^* := \{ S: x(S) \neq 0\}$.
		
		We will now map this to a feasible solution to $\LP_{\atoms}$ and show that the objective value does not decrease. This will then complete the claim. Consider the following solution $\vec{y}$ defined as follows.
		\[
		y(a) = \sum\limits_{S \in S^*: a \in S} x(S).
		\]
		We will now show that $\vec{y}$ is a feasible solution to the polytope $\mathcal{P}$. From the definition of $\vec{y}$, we can write it as $\vec{y} = \sum\limits_{S \in S^*} x(S) \times \mathbb{I}[S]$. Here, $\I[S]$ is a binary vector, such that it has $1$ at position $a$ if and only if atom $a$ is present in set $S$. Hence, $\vec{y}$ is a point in the polytope since it can be written as convex combination of its vertices.\\
		
		Now, we will show that, $\vec{y}$ also satisfies the resource consumption constraint.
\begin{align*}
\vec{C(j)}\cdot\vec{y}
    = \sum_{a\in\mA} C(a, j) \sum_{S \in S^*: a \in S} x(S) 
    =\sum_{S \in S^*} \sum_{a \in S} C(a, j) x(S) 
    = \sum_{S \in S^*} C(S, j) x(S)
    \leq \B/T.
\end{align*}

		The last inequality is because in the optimal solution, the x value corresponding to subset $S^*$ is 1 while rest all are 0. We will now show that $\vec{y}$ produces an objective value at least as large as $\vec{x}$.
		
\begin{align*}
\LPOptAtoms
    &= \vec{\mu}\cdot\vec{y^*}
    \geq \vec{\mu}\cdot\vec{y}
    = \sum_{a=1}^n \mu(a) \sum_{S \in S^*: a \in S} x(S) \\
	&=\sum_{S \in S^*} \sum_{a \in S}\mu(a) x(S)
    = \sum_{S \in S^*} \mu(S) x(S) \\
    &= \LPOptSubsets(\B).
\end{align*}
		
		Now we will prove the first inequality. We will assume the ``clean event'' that $\vec{\mu}^+_t \geq \vec{\mu}$ and $\vec{C}^-_t \leq \vec{C}_t$ for all $t$. Hence, the inequality holds with probability at least $1- \delta$.
		
		Consider a time $t$. Given an optimal solution $\vec{x^{\ast}}$ to $\LP_{\atoms}$ we will show that this is feasible to $\LP_{\ALG, t}$. Note that, $\vec{x^{\ast}}$ satisfies the constraint set $\vec{x} \in \mathcal{P}$ since that is same for both $\LP_{\ALG, t}$ and $\LP_{\atoms}$. Now consider the constraint $\vec{C_t^-(j)}\cdot\vec{x} \leq \frac{\B}{T}$. Note that $C_t^-(a, j) \leq C(a, j)$. Hence, we have that
		$\vec{C_t^-(j)}\cdot\vec{x^{\ast}} \leq \vec{C(j)}\cdot\vec{x^{\ast}} \leq \frac{\B}{T}$.
		The last inequality is because $\vec{x^{\ast}}$ is a feasible solution to $\LP_{\atoms}$.
		
		Now consider the objective function. Let $\vec{y^{\ast}}$ denote the optimal solution to $\LP_{\ALG, t}$.\\
		$\LPOptAlg[t] = \vec{\mu_t^+}\cdot\vec{y^{\ast}} \geq \vec{\mu_t^+}\cdot\vec{x^{\ast}} \geq \vec{\mu}\cdot\vec{y^{\ast}} = \LPOptAtoms$.
\end{proof}

Hence, combining Claim~\ref{clm:oneMinusEpsilon} and Claim~\ref{clm:armsubsets}, we obtain Lemma~\ref{clm:algopt}.

\OMIT{ 
	Now we will consider the following LP used by the algorithm at round $t$ and show a claim relating it to \ref{LP:arms}.
	\begin{equation*}
	\begin{array}{llll}
	\label{LP:algt}
	\tag{$\LP_{\ALG}$}
	\text{maximize}  & \vec{\mu_t^+}\cdot\vec{x}& &\\
	\text{subject to}& \vec{C_t^-}\cdot\vec{x} \preccurlyeq \vec{\B/T} & \vec{x} \in \mathcal{P} & \vec{x} \in [0,1]^n
	\end{array}
	\end{equation*}
	Let $\LPOptAlg$ denote the optimal value to \ref{LP:algt}. We then have,
	\begin{claim}
		\label{clm:armALG}
		$\LPOptAlg \geq \LPOptAtoms$
	\end{claim}
} 

\subsection{Negative correlation and concentration bounds}

Our analysis relies on several facts about negative correlation and concentration bounds. First, we argue that property \eqref{eq:neg-cor-defn} in the definition of negative correlation is preserved under a specific linear transformation:

\begin{claim}\label{cl:prelims-negCor-transform}
Suppose $(X_1,X_2 \LDOTS X_m)$ is a family of negatively correlated random variables with support $[0,1]$. Fix numbers $\lambda_1, \lambda_2 \LDOTS\lambda_m\in[0,1]$.
Consider two families of random variables:
\[
    \mF^+ = \left( \frac{1+\lambda_i (X_i-\E[X_i])}{2}:\, i\in[m]\right)
\quad\text{and}\quad
    \mF^-=\left(\frac{1-\lambda_i(X_i-\E[X_i])}{2}:\, i\in[m] \right). \]
Then both families satisfy property \eqref{eq:neg-cor-defn}.
\end{claim}

\begin{proof}
Let us focus on family $\mF^+$; the proof for family $\mF^-$ is very similar.

Denote $\mu_i = \E[X_i]$ and $Y_i:= (1+\lambda_i(X_i-\mu_i))/2$
and
$z_i:= (1-\lambda_i \mu_i)/2$
for all $i \in [m]$. Note that
    $Y_i = \lambda_i X_i/2 + z_i$ and $z_i \geq 0, X_i \geq 0$.
Fix a subset $S \subseteq [m]$. We have,
\begin{align*}
\E\left[\prod_{i \in S} Y_i\right]
    &= \E\left[\sum_{T \subseteq S}\prod_{i \in T}(\lambda_i X_i/2)\prod_{j \in S \setminus T}z_j \right]
    &\text{by Binomial Theorem}\\
&= \sum_{T \subseteq S} \E\left[\prod_{i \in T} (\lambda_i X_i/2) \right] \prod_{j \in S \setminus T} z_j &\\
&\leq \sum_{T \subseteq S}  \prod_{i \in T} (\lambda_i \mu_i/2) \prod_{j \in S \setminus T} z_j&\text{\eqref{eq:neg-cor-defn} invariant under non-negative scaling, $X_i$ neg. correlated} \\
&= \prod_{i \in S}( (1-\lambda_i \mu_i)/2 + \lambda_i \mu_i/2)
    &\text{by Binomial Theorem} \\
&= (\tfrac12)^{|S|}
= \prod_{i \in S} \E[Y_i] & \qedhere
\end{align*}
\end{proof}

\OMIT{ 
Now we will show that for any sequence $\{\lambda_i: 0 \leq \lambda_i \leq 1, i \in [m]\}$, the random variables $\{ (1 - \lambda_i (X_i - \E[X_i]))/2 : i \in [m]\}$ satisfy \eqref{eq:neg-cor-defn}. Also, the symmetric properties \eqref{eq:neg-cor-defn} and \eqref{eq:neg-cor-defn-2} imply that a sequence of random variables $\{X_i : i \in [m]\}$ are negatively correlated if and only if the sequence $\{ 1-X_i : i \in [m] \}$ are negatively correlated.

Define $A_i:= ((1-\lambda_i) + \lambda_i \mu_i)/2$ and $B_i := \lambda_i(1-X_i)/2$. Note that from the argument we made above we have, $B_i = \lambda Z_i/2$ where the sequence $\{Z_i : i \in [m] \}$ are negatively correlated. Since $0 \leq \lambda_i \leq 1$ for every $i \in [m]$ and that $0 \leq X_i \leq 1$ and hence $0 \leq \mu_i \leq 1$, we have that $A_i \geq 0$ and $B_i \geq 0$. Hence, the random variables of interest $\{ (1 - \lambda_i (X_i - \E[X_i]))/2 : i \in [m]\}$ = $\{ A_i + B_i : i \in [m] \}$.

Consider a subset $S \subseteq [m]$. We have,
\begin{align*}
\E[\prod_{i \in S} (B_i + A_i)] &= \E[\sum_{T \subseteq S}\prod_{i \in T}B_i\prod_{j \in S \setminus T}A_j]&\text{Applying Binomial Theorem on expression inside expectation}\\
&= \sum_{T \subseteq S} \E[\prod_{i \in T} B_i ] \prod_{j \in S \setminus T} A_j &\text{From Linearity of Expectation, $A_i \geq 0$ }\\
&\leq \sum_{T \subseteq S} \prod_{i \in T} (\lambda_i(1-\mu_i)/2) \prod_{j \in S \setminus T} A_j &\text{$\lambda_i \geq 0$, $Z_i$ are neg. correlated with \eqref{eq:neg-cor-defn} invariant under scaling} \\
&= \prod_{i \in S}( \lambda_i (1-\mu_i)/2 + A_i) = 1/2 &\text{From Binomial Theorem}\\
& = \prod_{i\in S} \E[A_i + B_i]&	\qedhere
\end{align*}
} 

Second, we extend Theorem~\ref{thm:prelim:additive} to a random process that evolves over time, and only assumes that property \eqref{eq:neg-corr-half} holds within each round conditional on the history.

\begin{theorem}\label{lemma:combining}
Let 	
	$\mZ_T = \{\zeta_{t,a}: a \in \mA, t \in [T]\}$ 	
be a family of random variables taking values in $[0,1]$. Assume random variables 	
	$\{\zeta_{t,a}: a \in \mA\}$ 	
\kaedit{satisfy property \eqref{eq:neg-cor-defn}} given $\mZ_{t-1}$ and have expectation $\tfrac12$ given $\mZ_{t-1}$, for each round $t$. Let
    $Z = \tfrac{1}{nT} \; \sum_{a\in \mA, t\in [T]} \zeta_{t,a} $
be the average. Then for some absolute constant $c$,
\begin{align}\label{eq:lemma:combining}
    \Pr[Z \geq \tfrac12 + \eta] \leq c \cdot e^{-2m\eta^2}
    \qquad (\forall \eta>0).
\end{align}
\end{theorem}

\begin{proof}
We prove that family $\mZ_t$ satisfies property \eqref{eq:neg-corr-half}, and then invoke Theorem~\ref{thm:prelim:additive}. Let us restate property \eqref{eq:neg-corr-half} for the sake of completeness:
\begin{align}\label{eq:lemma:combining-proof}
\E\left[ \prod_{(t, a) \in S} \zeta_{t,a} \right] \leq 2^{-|S|}
    \quad \text{for any subset $S \subseteq \mZ_T$.}
\end{align}
Fix subset $S\subset \mZ_T$. Partition $S$ into subsets
    $S_t = \{\zeta_{t,a} \in \mZ_T \cap S\}$,
for each round $t$. Fix round $\T$ and denote
\[  G_\T = \prod_{t \in [\T]}H_t,
    \;\text{where}\;
    H_t = \prod_{a \in S_t} \zeta_{t,a}.\]
    We will now prove the following statement by induction on $\tau$:
\begin{align}
\label{eq:combining}
\E[G_\T] \leq 2^{-k_\T},
\;\text{where}\;
 k_\T = \sum_{t \in [\T]}\, |S_t|.
\end{align}
The base case is when $\T = 1$. Note that $G_\T$ is just the product of elements in set $\zeta_1$ and they are negatively correlated from the premise. Therefore we are done. Now for the inductive case of $\tau \geq 2$,
\begin{align}
\label{pf:change}
\E[H_\T| \mathcal{Z}_{\T-1}] &\leq \prod_{a \in S_\T} \E[\zeta_{\T,a} | \mathcal{Z}_{\T-1}] & \text{ From \kaedit{property \eqref{eq:neg-cor-defn}} in the conditional space}\\
\label{pf:final}
&\leq 2^{-|S_\T|}&\text{From assumption in Lemma~\ref{lemma:combining}}
\end{align}

Therefore, we have
\begin{align*}
\E[G_\T] &= \E[\E[G_{\T-1}H_\T | \mathcal{Z}_{\T-1}]]&\text{Law of iterated expectation}\\
&=\E[G_{\T-1} \E[H_\T | \mathcal{Z}_{\T-1}]] &\text{Since $G_{\T-1}$ is a fixed value conditional on $\mathcal{Z}_{\T-1}$}\\
&\leq 2^{-|S_\T|} \E[G_{\T-1}]&\text{From \refeq{pf:final}}\\
&\leq 2^{-k_\T} &\text{From inductive hypothesis}
\end{align*}	
This completes the proof of \refeq{eq:combining}. We obtain \refeq{eq:lemma:combining-proof} for $\T = T$.
\end{proof}

Third, we invoke \refeq{eq:conf-rad-prop} for rewards and resource consumptions:

\begin{lemma}
	\label{appx:devanurb3}
With probability at least $1- e^{-\Omega(\alpha)}$, we have the following:
	\begin{align}
	\begin{split}
	|\hat{\mu}_t(a) - \mu_t(a)| &\leq 2 \rad(\hat{\mu}_t(a), N_t(a) + 1)\\
	\forall j \in [d] \quad |\hat{C}_t(a, j) - C_t(a, j)| &\leq 2 \rad(\hat{C}_t(a, j), N_t(a) + 1).
	\end{split}
	\end{align}
\end{lemma}

Fourth, we use a concentration bound from prior work which gets sharper when the expected sum is very small, and does not rely on independent random variables:

\begin{theorem}[\citet{DynPricing-ec12}]
	\label{prelim:babioff}
	Let $X_1, X_2, \ldots, X_m$ denote a set of $\{0, 1\}$ random variables. For each $t$, let $\alpha_t$ denote the multiplier determined by random variables $X_1, X_2, \ldots, X_{t-1}$. Let $M = \sum_{t=1}^m M_t$ where $M_t = \E[X_t | X_1, X_2, \ldots, X_{t-1}]$. Then for any $b \geq 1$, we have the following with probability at least $1- m^{-\Omega(b)}$:
	$$|\sum_{t=1}^m \alpha_t (X_t - M_t) | \leq b (\sqrt{M \log m} + \log m)$$
\end{theorem}

\subsection{Analysis of the ``clean event"}

Let us set up several events, henceforth called \emph{clean events}, and prove that they hold with high probability. Then the remainder of the analysis can proceed conditional on the intersection of these events. The clean events are similar to the ones in \cite{agrawal2014bandits}, but are somewhat ``stronger", essentially because our algorithm has access to per-atom feedback and our analysis can use the negative correlation property of the $\RRS$.

In what follows, it is convenient to consider a version of \SemiBwK in which the algorithm does not stop, so that we can argue about what happens w.h.p. if our algorithm runs for the full $T$ rounds. Then we show that our algorithm does indeed run for the full $T$ rounds w.h.p.

Recall that $\vec{x_t}$ be the optimal fractional solution obtained by solving the LP in round $t$. Let $\vec{Y_t}\in \{0,1\}^n$ be the random binary vector obtained by invoking the $\RRS$ (so that the chosen action $S_t\in \mF$ corresponds to a particular realization of $\vec{Y_t}$, interpreted as a subset). Let
    $\mG_t := \{\vec{Y_{t'}} : \forall t' \leq t  \}$
denote the family of RRS realizations up to round $t$.

\subsubsection{``Clean event" for rewards}

For brevity, for each round $t$ let $\vec{\mu}_t = (\mu_t(a):\; a\in A)$ be the vector of realized rewards, and let
$r_t := \mu_t(S_t) = \vec{\mu_t} \cdot \vec{Y_t}$
be the algorithm's reward at this round.

\begin{lemma} \label{lemma:rewLPWarmup}
	Consider \SemiBwK without stopping. Then with probability at least
	$1-nT\; e^{-\Omega(\alpha)}$:
	\begin{equation*}
	|\sum_{t\in [T]}\; r_t - \sum_{t\in [T]}\; \vec{\mu_t^+}\cdot\vec{x_t}|
	\leq O\left( \sqrt{\alpha n \sum_{t\in [T]}\; r_t} + \sqrt{\alpha n T} + \alpha n \right).
	\end{equation*}
\end{lemma}

\begin{proof}
	We prove the Lemma by proving the following three high-probability inequalities.
	
	With probability at least $1-nT\; e^{-\Omega(\alpha)}$:
	the following holds:
	\begin{align}
	\label{w:part1}
	 |\sum_{t\in [T]}\; r_t - \sum_{t\in [T]}\; \vec{\mu}\cdot\vec{Y_t}|
	&\leq 3nT \rad \left(\frac{1}{nT} \sum_{t\in [T]}\;
	\vec{\mu_t^+}\cdot\vec{x_t}~,~nT \right) \\
	\label{w:part2}
	 |\sum_{t\in [T]}\;\vec{\mu}\cdot\vec{Y_t} - \vec{\mu_t^+}\cdot\vec{Y_t}|
	&\leq  12 \sqrt{\alpha n \left(\sum_{t\in [T]}\;
		\vec{\mu_t^+}\cdot\vec{x_t}\right)} + 12 \sqrt{\alpha} n + 12 \alpha n \\
	\label{w:part3}
	 |\sum_{t\in [T]}\; \vec{\mu_t^+}\cdot\vec{Y_t} -  \vec{\mu_t^+}\cdot\vec{x_t}|
	&\leq   \sqrt{\alpha nT}.
	\end{align}
	
	We will use the properties of $\RRS$ to prove \refeq{w:part3}. Proof of \refeq{w:part2} is similar to \cite{agrawal2014bandits}, while proof of \refeq{w:part1} follows immediately from the setup of the model. Using the parts \eqref{w:part1} and \eqref{w:part3} we can now find an appropriate upper bound on $\sqrt{\sum_{t\in [T]}\; \vec{\mu_t^+}\cdot\vec{x_t}}$ and using this upper bound, we prove Lemma \ref{lemma:rewLPWarmup}.

	\xhdrem{Proof of \refeq{w:part1}.}
	Recall that $r_t = \vec{\mu_t}\vec{Y_t}$. Note that, $\E[\vec{\mu_t}\vec{Y_t}] = \vec{\mu}\vec{Y_t}$ when the expectation is taken over just the independent samples of $\mu$. By Theorem \ref{prelim:babioff}, with probability $1-e^{-\Omega(\alpha)}$ we have:	
	\begin{align*}
	|\sum_{t\leq T} r_t - \sum_{t\leq T} \vec{\mu}\cdot\vec{Y_t}| & \leq 3nT \rad\left(\frac{1}{nT} \sum_{t\leq T} \vec{\mu}\cdot\vec{Y_t}~,~nT \right)\\
	&\leq 3nT \rad\left(\frac{1}{nT} \sum_{t\leq T} \vec{\mu_t^+}\cdot\vec{Y_t}~,~nT \right)\\
	&\leq 3nT \rad\left(\frac{1}{nT} \sum_{t\leq T} \vec{\mu_t^+}\cdot\vec{x_t}~,~nT \right).
	\end{align*}
	
	The last inequality is because $Y_t$ is a feasible solution to \ref{lp:template}.
	
	\xhdrem{Proof of \refeq{w:part2}.}
	For this part, the arguments similar to \cite{agrawal2014bandits} follow with some minor adaptations. For sake of completeness we describe the full proof. Note that we have,
	\[
	 |\sum_{t\leq T} \vec{\mu}\cdot\vec{Y_t} - \vec{\mu_t^+}\cdot\vec{Y_t}|
	\leq \sum_{a=1}^n |\sum_{t\leq T}\mu(a) Y_t(a) - \mu_t^+(a) Y_t(a)|.
	\]
	
	Now, using Lemma \ref{appx:devanurb3} in Appendix, we have that with probability $1-nT e^{-\Omega(\alpha)}$
	\[
	|\sum_{t\leq T}\mu(a) Y_t(a) - \mu_t^+(a) Y_t(a)| \leq 12 \sum_{t\leq T} \rad(\mu(a), N_t(a) + 1).
	\]
	Hence, we have
	\begin{align*}
	\sum_{a=1}^n |\sum_{t\leq T}\mu(a) Y_t(a) - \mu_t^+(a) Y_t(a)| &= 12 \sum_{a\in\mA} \sum_{r=1}^{N_T(a) + 1} \rad(\mu(a), r)\\
	&\leq 12 \sum_{a\in\mA} (N_T(a) + 1) \rad(\mu(a), N_T(a) + 1)\\
	&\leq 12 \sqrt{\alpha n \left(\vec{\mu}\cdot\vec{(N_T + 1)}\right)} + 12 \alpha n.
	\end{align*}
	The last inequality is from the definition of $\rad$ function and using the Cauchy-Swartz inequality. Note that $\vec{\mu N_T} = \sum_{t\leq T} \vec{\mu}\cdot\vec{Y_t}$. Also, since we have with probability $1- e^{-\Omega(\alpha)}$, $\mu(a) \leq \mu_t^+(a)$, we have,
	\[
	 12 \sqrt{\alpha n \left(\vec{\mu}\cdot\vec{(N_T + 1)}\right)} + 12 \alpha n \leq
	12 \sqrt{\alpha n \left( \sum_{t\leq T}\vec{\mu_t^+}\cdot\vec{Y_t} \right)} + 12 \sqrt{\alpha} n + 12 \alpha n.
	\]
	
	Finally note that $\vec{Y_t}$ is a feasible solution to the semi-bandit polytope $\mathcal{P}$. Hence, we have that
	\[
	\vec{\mu^+_t}\cdot\vec{Y_t} \leq \vec{\mu^+_t}\cdot\vec{x_t}.
	\]
	Hence,	
	\[
	 12 \sqrt{\alpha n \left( \sum_{t\leq T}\vec{\mu_t^+}\cdot\vec{Y_t} \right)} + 12 \sqrt{\alpha} n + 12 \alpha n \leq 12 \sqrt{\alpha n \left( \sum_{t\leq T}\vec{\mu_t^+}\cdot\vec{x_t} \right)} + 12 \sqrt{\alpha} n + 12 \alpha n.
	\]
	
	\xhdrem{Proof of \refeq{w:part3}:}
Recall that for each round $t$, the UCB vector $\vec{\mu_t^+}$ is determined by the random variables
    $\mG_{t-1} = \{\vec{Y_{t'}} : \forall t' < t  \}$.
Further, conditional on a realization of $\mG_{t-1}$, the random variables $\{Y_{t}(a): a \in \mA\}$ are \kaedit{negatively correlated from the property of RRS}. \kaedit{Let  $\tilde{\zeta}_t(a) := \mu_t^+(a)\, Y_t(a)$, $a\in \mA$. Note that we have
    $\E[\tilde{\zeta}_t(a) | \mG_{t-1}] = \mu_t^+(a)\, x_t(a)$. }
Define
\[ \zeta_t(a) := \frac{1+\mu_t^+(a)\, Y_t(a) -\mu_t^+(a)\,x_t(a)}{2}.\]
\kaedit{From Claim \ref{cl:prelims-negCor-transform}, we have that $\{\zeta_t(a): a \in \mA\}$ conditioned on $\mG_{t-1}$ satisfy 	\eqref{eq:neg-cor-defn}.} Further,
    $\E[\zeta_t(a) | \mG_{t-1}] = \tfrac12$. 	
Therefore, the family
	$\{\zeta_t(a) : t \in [T], a \in \mA\}$
satisfies the assumptions in Theorem~\ref{lemma:combining} and hence satisfies \refeq{eq:lemma:combining} for some absolute constant $c$. Plugging back the $\tilde{\zeta}_t(a)$'s, we obtain an upper-tail concentration bound:
	\begin{align*}
	\Pr\left[ \; \frac{1}{nT}(\sum_{t=1}^{T}\sum_{a\in\mA}\;
	\tilde{\zeta}_t(a) -\mu_t^+(a)\,x_t(a)) \geq \eta \;\right]
	\leq c \cdot e^{-2nT\eta^2}.
	\end{align*}
	
To obtain a corresponding concentration bound for the lower tail, we apply a similar argument to
    \[ \zeta'_t(a) = \frac{1+\mu_t^+(a)\,x_t(a)-\tilde{\zeta}_t(a)}{2}.\]
Once again from Claim~\ref{cl:prelims-negCor-transform}, we have that $\{\zeta'_t(a): a \in \mA\}$ conditioned on $\mG_{t-1}$ \kaedit{satisfy \eqref{eq:neg-cor-defn}}. The family
	$\{\zeta'_t(a) : t \in [T], a \in \mA\}$
satisfies the assumptions in Theorem~\ref{lemma:combining} and hence satisfies \refeq{eq:lemma:combining}. Plugging back the $\tilde{\zeta}_t(a)$'s, we obtain a lower-tail concentration bound:
	\begin{align*}
	\Pr\left[ \; \frac{1}{nT}(\sum_{t=1}^{T}\sum_{a\in\mA}\;
	 \mu_t^+(a)x_t(a) - \tilde{\zeta}_t(a)) \geq \eta \;\right]
	\leq c \cdot e^{-2nT\eta^2}.
	\end{align*}
	
	Combining these two we have,
	\begin{equation}
	\label{eq:rewChernoff}
	\Pr\left[ \; \frac{1}{nT}|\sum_{t=1}^{T}\sum_{a\in\mA}\;
	\mu_t^+(a)Y_t(a) -\mu_t^+(a)x_t(a)| \geq \eta \;\right]
	\leq 2\,c \cdot e^{-2nT\eta^2}.
	\end{equation}
	
Hence setting $\eta = \sqrt{\frac{\alpha}{nT}}$, we obtain \refeq{w:part3} with probability at least $1-e{^{-\Omega(\alpha)}}$.
	
\xhdrem{Combining Eq. \eqref{w:part1}, \eqref{w:part2} and \eqref{w:part3}}
Let us denote $H := \sqrt{\sum_{t\in [T]}\;\vec{\mu_t^+}\cdot\vec{x_t}}$. Adding the three equations we get
	\begin{equation}
	\label{eq:Heqwarmup}
	 |\sum_{t \in [T]}\;r_t - H^2| \leq \sqrt{\alpha}H + \alpha + \sqrt{\alpha n}H + O(\alpha n) + \sqrt{\alpha n T}
	\end{equation}
	
	Rearranging and solving for $H$, we have
	\[
	 H \leq \sqrt{\sum_{t \in [T]}\;r_t} + O(\sqrt{\alpha n}) + (\alpha n T)^{1/4}
	\]
	
	Plugging this back into \refeq{eq:Heqwarmup}, we get Lemma~\ref{lemma:rewLPWarmup}.
\end{proof}

\OMIT{
	\begin{lemma}\label{clm:rewLP}
		Consider \SemiBwK without stopping. Then with probability at least
		$1-nT\; e^{-\Omega(\alpha)}$:
		\begin{equation*}
		 |\sum_{t\in [T]}\; r_t - \sum_{t\in [T]}\; \vec{\mu_t^+}\cdot\vec{x_t}|
		\leq O\left( \sqrt{\alpha n \sum_{t\in [T]}\; r_t} + \alpha n \right).
		\end{equation*}
	\end{lemma}
	
	\begin{proof}{(\emph{Sketch})}
		As before, we will split the proof of this lemma by proving the following three equations.
		
		With probability at least $1-nT\; e^{-\Omega(\alpha)}$:
		the following holds:
		\begin{align}
		\label{part1}
		 |\sum_{t\in [T]}\; r_t - \sum_{t\in [T]}\; \vec{\mu}\cdot\vec{Y_t}|
		&\leq   3nT \rad\left(\frac{1}{nT} \sum_{t\in [T]}\;
		\vec{\mu_t^+}\cdot\vec{x_t}~,~nT \right) \\
		\label{part2}
		 |\sum_{t\in [T]}\;\vec{\mu}\cdot\vec{Y_t} - \vec{\mu_t^+}\cdot\vec{Y_t}|
		&\leq  12 \sqrt{\alpha n \left(\sum_{t\in [T]}\;
			\vec{\mu_t^+}\cdot\vec{x_t}\right)} + 12 \sqrt{\alpha} n + 12 \alpha n \\
		\label{part3}
		 |\sum_{t\in [T]}\; \vec{\mu_t^+}\cdot\vec{Y_t} -  \vec{\mu_t^+}\cdot\vec{x_t}|
		&\leq   3nT \rad\left(\frac{1}{nT}\sum_{t\in [T]}\;
		\vec{\mu_t^+}\cdot\vec{x_t}~,~ nT\right)
		\end{align}
		
		Note that \refeq{part1} and \refeq{part2} are the same as before, while we get a sharper bound in \refeq{part3}.
		We will use the properties of $\RRS$ to prove \refeq{part3}. Proof of \refeq{part2} is similar to \cite{agrawal2014bandits}, while proof of \refeq{part1} follows immediately from the setup of the model. Using the parts \eqref{part1} and \eqref{part3} we can now find an appropriate upper bound on $\sqrt{\sum_{t\in [T]}\; \vec{\mu_t^+}\cdot\vec{x_t}}$ and using this upper bound, we prove Lemma \ref{clm:rewLP}.

		We prove \refeq{part3} as follows. Using similar arguments as the warm-up Lemma~\ref{lemma:rewLPWarmup}, we obtain \refeq{eq:rewChernoff}. Additionally, we make a few more observations. In particular, we have the following.
		
		\begin{align*}
		&\mu_t^+(a) x_t(a) &\\
		&= \E[\mu_t^+(a)Y_t(a)|\mG_{t-1}]&\\
		&= \mu_t^+(a)\E[Y_t(a)|\mG_{t-1}]&\text{ Since, $\mu_t^+(a)$ is a constant in the conditional space}\\
		& = \mu_t^+(a) \Pr[Y_t(a) = 1 | \mG_{t-1}] &\text{Since, $Y_t(a)$ is a 0-1 random variable}\\
		& \geq \mu_t^+(a) \Pr[Y_t(a) = 1 | \mG_{t-1}, \{Y_t(a'):\forall a' < a\}]& \text{$Y_t$'s are neg. correlated in the conditional space $\mG_{t-1}$}\\
		&=\E[\mu_t^+(a)Y_t(a)|\mG_{t-1}, \{Y_t(a'):\forall a' < a\}]&
		\end{align*}
		
		Taking sums over $t$ from $1$ to $T$ and $a$ from $1$ to $n$, we have
		
		$$\sum_{t=1}^{T}\sum_{a\in\mA} \mu_t^+(a)x_t(a) \geq \sum_{t=1}^{T}\sum_{a\in\mA} \E[\mu_t^+(a)Y_t(a)|\mG_{t-1}, \{Y_t(a'):\forall a' < a\}]$$
		
		We see that random variables
		$\{\tilde{\zeta}_t(a) : t \in [T], a \in \mA\}$, $\{\mu_t^+(a)x_t(a) : t \in [T], a \in \mA\}$
		satisfy the properties in Theorem \ref{prelim:radTheorem}(b), with
		$X = \frac{1}{nT}(\sum_{t=1}^{T}\sum_{a\in\mA}\;\tilde{\zeta}_t(a))$, $\hat{\mu} = \frac{1}{nT}(\sum_{t=1}^{T}\sum_{a\in\mA}\; \mu_t^+(a)x_t(a))$.
		By Theorem \ref{prelim:radTheorem}(b), these random variables satisfy
		\refeq{eq:conf-rad-prop}
		with $\mu$ replaced with $\hat{\mu}$.
		We conclude that with probability $1-\Omega(\exp(-\alpha))$ we have \eqref{part3}.
	\end{proof}
}

\vspace{-2mm}

\subsubsection{``Clean event" for resource consumption}

We define a similar ``clean event" for consumption of each resource $j$. By a slight abuse of notation, for each round $t$ let
$\vec{C_t(j)} = (C_t(a,j):\; a\in \mA)$
be the vector of realized consumption of resource $j$. Let $\chi_t(j)$ denote algorithm's consumption for resource $j$ at round $t$ (\ie $\chi_t(j) = \vec{C_t(j)} \cdot \vec{Y_t}$).

\begin{lemma}\label{clm:constLP}
	Consider \SemiBwK without stopping. Then with probability at least
	$1-nT\; e^{-\Omega(\alpha)}$:
	\begin{equation*}
	 \forall j \in [d] \quad |\sum_{t\in [T]}\; \chi_t(j) - \sum_{t\in [T]}\; \vec{C_t^-(j)}\cdot\vec{x_t}| \leq \sqrt{\alpha n \B} + \alpha n + \sqrt{\alpha n T}.
	\end{equation*}
\end{lemma}

\begin{proof}
The proof is similar to Lemma~\ref{lemma:rewLPWarmup}. We will split the proof into following three equations. Fix an arbitrary resource $j \in [d]$. With probability at least $1-nT e^{-\Omega(\alpha)}$ the following holds:
	
	\begin{equation}
	\label{c:part1}
	|\sum_{t\leq T} \chi_t(j) - \sum_{t\leq T} \vec{C(j)}\cdot\vec{Y_t}| \leq 3nT \rad\left(\frac{1}{nT} \sum_{t\leq T} \vec{C(j)}\cdot\vec{Y_t}~,~nT \right).
	\end{equation}
	
	\begin{equation}
	\label{c:part2}
	|\sum_{t\leq T} \vec{C(j)}\cdot\vec{Y_t} - \vec{C_t^-(j)}\cdot \vec{Y_t}| \leq 12 \sqrt{\alpha n \left(\sum_{t\leq T} \vec{C(j)}\cdot \vec{Y_t} \right)} + 12 \sqrt{\alpha }n + 12 \alpha n.
	\end{equation}
	
	\begin{equation}
	\label{c:part3}
	|\sum_{t\leq T} \vec{C_t^-(j)}\cdot\vec{Y_t} - \vec{C_t^-(j)}\cdot \vec{x_t}| \leq \sqrt{\alpha n T}.
	\end{equation}
	
	Using the parts \ref{c:part1}, \ref{c:part2} and \ref{c:part3} we can find an upper bound on $\sqrt{\sum_{t\leq T} \vec{C_t(j)}\cdot\vec{Y_t}}$. Hence, combining Lemmas \ref{c:part1}, \ref{c:part2} and \ref{c:part3} with this bound and taking an Union Bound over all the resources, we get Lemma \ref{clm:constLP}.

	\xhdrem{Proof of \refeq{c:part1}.}
	We have that $\{C_t(a, j) : a \in \mA\}$ is a set of independent random variables over a probability spacee $C_{\Omega}$. Note that, $\E_{C_{\Omega}} C_t(a, j) Y_t(a) = C(a, j)Y_t(a)$. Hence, we can invoke Theorem \ref{prelim:babioff} on \emph{independent random variables} to get with probability $1- nT e^{-\Omega(\alpha)}$
	
	\begin{align*}
	|\sum_{t\leq T} \chi_t(j) - \sum_{t\leq T} \vec{C(j)}\cdot\vec{Y_t}| & \leq 3nT \rad\left(\frac{1}{nT} \sum_{t\leq T} \vec{C(j)}\cdot\vec{Y_t}~,~nT \right).
	\end{align*}
	
	\xhdrem{Proof of \refeq{c:part2}.}
	This is very similar to proof of \ref{w:part2} and we will skip the repetitive parts. Hence, we have with probability $1-nT e^{-\Omega(\alpha)}$
\begin{align*}
	 |\sum_{t\leq T} \vec{C(j)}\cdot \vec{Y_t} - \vec{C_t^-(j)}\cdot\vec{Y_t}|
	&\leq 12 \sqrt{\alpha n (\vec{C(j)}\cdot\vec{(N_T + 1)})} + 12 \alpha n \\
	 &\leq 12 \sqrt{\alpha n \left(\sum_{t\leq T} \vec{C(j)}\cdot\vec{Y_t} \right)} + 12 \sqrt{\alpha }n + 12 \alpha n.
\end{align*}
	
	\xhdrem{Proof of \refeq{c:part3}.}
    Recall that for each round $t$ and each resource $j$, the LCB vector $\vec{C_t^-(j)}$ is determined by the random variables
	$\mG_{t-1} = \{\vec{Y_{t'}} : \forall t' < t\}$.
Similar to the proof of \refeq{w:part3}, random variables
    $\{Y_t(a) : a \in \mA \}$ \kaedit{obtained from the RRS}
are negatively correlated given $\mG_{t-1}$. \kaedit{As before define
    $\tilde{\zeta}_t(a) = C_t^-(a)\,Y_t(a)$, $a\in \mA$.}
We have that $\E[\zeta_t(a)~|~\mG_{t-1}] = C_t^-(a)\,x_t(a)$.

By Claim~\ref{cl:prelims-negCor-transform}, random variables
    \[ \zeta_t(a) = \frac{1+\tilde{\zeta}_t(a)-C_t^-(a)\,x_t(a)}{2}, \; a\in\mA \]
\kaedit{satisfy \eqref{eq:neg-cor-defn}}, given $\mG_{t-1}$. We conclude that family
	$\{\zeta_t(a) : t \in [T], a \in \mA\}$
satisfies the assumptions in Theorem \ref{lemma:combining}, and therefore satisfies \refeq{eq:lemma:combining} for some absolute constant $c$.
 Therefore, we obtain an upper-tail concentration bound for $\tilde{\zeta}_t(a)$'s:
	\begin{align*}
	\Pr\left[ \; \frac{1}{nT}(\sum_{t=1}^{T}\sum_{a\in\mA}\tilde{\zeta}_t(a) -C_t^-(a)\,x_t(a)) \geq \eta \;\right]
	\leq c \cdot e^{-2nT\eta^2}.
	\end{align*}

To obtain a corresponding concentration bound for the lower tail, we apply a similar argument to
\[ \zeta'_t(a) = \frac{1+C_t^-(a)\,x_t(a)-\tilde{\zeta}_t(a)}{2}.\]
Once again, invoking Claim~\ref{cl:prelims-negCor-transform} we have that $\{\zeta'_t(a): a \in \mA\}$ conditioned on $\mG_{t-1}$ \kaedit{satisfy \eqref{eq:neg-cor-defn}}. Thus, family
	$\{\zeta_t(a) : t \in [T], a \in \mA\}$
satisfies the assumptions in Theorem \ref{lemma:combining}, and therefore satisfies \refeq{eq:lemma:combining}. We obtain:
	\begin{align*}
	\Pr\left[ \; \frac{1}{nT}(\sum_{t=1}^{T}\sum_{a\in\mA}C_t^-(a)\,x_t(a)-\tilde{\zeta}_t(a)) \geq \eta \;\right]
	\leq c \cdot e^{-2nT\eta^2}.
	\end{align*}
	
	Combing the two tails we have,
	\begin{equation}
	\label{eq:conChernoff}
	\Pr\left[ \; \frac{1}{nT}|\sum_{t=1}^{T}\sum_{a\in\mA}\;
	C_t^-(a)Y_t(a) -C_t^-(a)\,x_t(a)| \geq \eta \;\right]
	\leq 2\,c \cdot e^{-2nT\eta^2}.
	\end{equation}
	
	Once again, setting $\eta = \sqrt{\frac{\alpha}{nT}}$, we obtain \refeq{c:part3} with probability at least $1-e{^{-\Omega(\alpha)}}$.
	
	\xhdrem{Proof of Lemma \ref{clm:constLP}.}
	Denote $G = \sqrt{\sum_{t\leq T} \vec{C(j)}\cdot\vec{Y_t}}$. From Equation \ref{c:part1}, \ref{c:part2} and \ref{c:part3}, we have that $G^2 - 2 \Omega(\sqrt{\alpha n})G \leq \sum_{t\leq T} \vec{C_t^-(j)}\cdot\vec{x_t} + O(\alpha n) + \sqrt{\alpha nT}$. Note that $\sum_{t\leq T} \vec{C_t^-(j)}\cdot \vec{x_t} \leq \B$. Hence, $G^2 - 2 \Omega(\sqrt{\alpha n})G \leq \B + O(\alpha n) + \sqrt{\alpha nT}$. Hence, re-arranging this gives us $G \leq \sqrt{\B} + O(\sqrt{\alpha n}) + (\alpha nT)^{1/4}$. Plugging this back in Equations \ref{c:part1}, \ref{c:part2} and \ref{c:part3}, we get Lemma \ref{clm:constLP}.

\end{proof}

\subsection{Putting it all together}
Similar to \cite{agrawal2014bandits}, we will handle the hard constraint on budget, by choosing an appropriate value of $\eps$. We then combine the above Lemma on "rewards" clean event to compare the reward of the algorithm with that of the optimal value of LP to obtain the regret bound in Theorem~\ref{mainresult}. Additionally, we use the Lemma on "consumption" clean event to argue that the algorithm doesn't exhaust the resource budget before round $T$. Formally, consider the following.

Recall that from Lemma~\ref{clm:algopt}, we have $\LPOptAlg \geq \frac{1}{T}(1-\eps)\OPT$. Let us define the performance of the algorithm as $\ALG = \sum_{t\leq T} r_t$. From Lemma~\ref{lemma:rewLPWarmup}, that with probability at least $1-ndT\; e^{-\Omega(\alpha)}$
\begin{align*}
 \ALG & \geq  (1-\epsilon) \OPT - O(\sqrt{\alpha n \ALG}) - O(\alpha n) - \sqrt{\alpha nT}&\\
& \geq  (1-\epsilon) \OPT - O(\sqrt{\alpha n \OPT}) - O(\alpha n) -\sqrt{\alpha nT}&\text{(since $\ALG \leq \OPT$)}.
\end{align*}

Choosing
    $\epsilon = \sqrt{\frac{\alpha n}{B}} + \frac{\alpha n}{B} + \frac{\sqrt{\alpha n T}}{B}$
and using the assumption that $B > 3(\alpha n + \sqrt{\alpha nT})$,
we derive \refeq{eq:mainresult}. For any given $\delta$, we set $\alpha = \Omega(\log(\frac{ndT}{\delta}))$ to obtain a success probability of at least $1-\delta$.

Now we will argue that the algorithm does not exhaust the resource budget before round $T$ with probability at least $1-ndT\;e^{-\Omega(\alpha)}$. Note that for every resource $j \in [d]$,
\[ \sum_{t\leq T} \vec{C_t^-(j)}\cdot \vec{x_t} \leq (1-\epsilon)B.\]
Hence, combining this with Lemma \ref{clm:constLP}, we have
${\sum_{t\leq T} \vec{C_t(j)} \,\vec{Y_t} \leq (1-\epsilon)B + \epsilon B \leq B}.$

\OMIT{
	\xhdr{Proof of Lemma \ref{clm:rewLP}.}
		Denote $H = \sqrt{\sum_{t\leq T} \vec{\mu_t^+}\cdot\vec{x_t}}$. From \ref{part1}, \ref{part2} and \ref{part3}, we have that $H^2 - 2 \Omega(\sqrt{\alpha n})H \leq \sum_{t\leq T} r_t + O(\alpha n)$. Hence, re-arranging this gives us $H \leq \sqrt{\sum_{t\leq T} r_t} + O(\sqrt{\alpha n})$. Now plugging this back into equations \ref{part1}, \ref{part2} and \ref{part3} proves the Lemma \ref{clm:rewLP}.
}

\OMIT{
Similar to the rewards case, we have the following.
\begin{align*}
&C_t^-(a) x_t(a)&\\
&= \E[C_t^-(a)Y_t(a)|\mG_{t-1}]&\\
&= C_t^-(a)\E[Y_t(a)|\mG_{t-1}]&\text{ Since, $C_t^-(a)$ is a constant in the conditional space}\\
& = C_t^-(a) \Pr[Y_t(a) = 1 | \mG_{t-1}] &\text{Since, $Y_t(a)$ is a 0-1 random variable}\\
& \geq C_t^-(a) \Pr[Y_t(a) = 1 | \mG_{t-1}, \{Y_t(a'):\forall a' < a\}]& \text{$Y_t$'s are neg. correlated in the conditional space $\mG_{t-1}$}\\
&=\E[C_t^-(a)Y_t(a)|\mG_{t-1}, \{Y_t(a'):\forall a' < a\}]&
\end{align*}

Taking sums over $t$ from $1$ to $T$ and $a$ from $1$ to $n$, we have

\[
		\sum_{t=1}^{T}\sum_{a\in\mA} C_t^-(a)x_t(a) \geq \sum_{t=1}^{T}\sum_{a\in\mA} \E[C_t^-(a)Y_t(a)|\mG_{t-1}, \{Y_t(a'):\forall a' < a\}]
\]

Hence, from Theorem \ref{prelim:radTheorem}(b), with $X = \frac{1}{nT}(\sum_{t=1}^{T}\sum_{a\in\mA}\;\tilde{\zeta}_t(a))$, $\hat{\mu} = \frac{1}{nT}(\sum_{t=1}^{T}\sum_{a\in\mA}\; C_t^-(a)x_t(a))$ we have \refeq{eq:conf-rad-prop}
with $\mu$ replaced with $\hat{\mu}$. Hence, with probability $1-\Omega(\exp(-\alpha))$ we have,

	\begin{align*}
	|\sum_{t\leq T} \vec{C_t^-(j)}\cdot\vec{Y_t} -  \vec{C_t^-(j)}\cdot\vec{x_t}| &\leq nT \rad\left( \frac{1}{nT}\sum_{t\leq T} \vec{C_t^-(j)}\cdot\vec{Y_t}~,~ nT \right)\\
	&\leq nT \rad\left(\frac{1}{nT}\sum_{t\leq T} \vec{C(j)}\cdot\vec{Y_t}~,~ nT\right)\\
	\end{align*}
}

\OMIT{ 
\section{Proof sketch for an extension (Theorem~\ref{thm:naive})}
\label{sec:proof-extension}

	Here, we will give a sketch for proving the regret in \refeq{eq:thm:naive}. Note that, the $\RRS$ guarantees no form of concenteration among the atoms. Hence, we analyze it atom-by-atom and take an union bound across all the atoms. For the rewards clean event, we obtain with probability $1-n^2T \;e^{-\Omega(\alpha)}$

	\[
			|\sum_{t\leq T} r_t - \sum_{t\leq T} \vec{\mu_t^+}\cdot\vec{x_t}| \leq \sum_{a=1}^n O\left(\sqrt{\alpha n \sum_{t\leq T} r_t(a)}\right) + O(\alpha n^2)
	\]

	Translating this to the final regret calculation, we obtain

	\[
			OPT - ALG \leq O\left(\OPT \sqrt{\frac{\alpha n}{B}} + \sum_{a=1}^n \sqrt{\alpha n r_t(a)} + \alpha n^2 \right)
	\]

	Now, on the RHS we want to maximize $\sum_{a=1}^n \sqrt{r_t(a)}$ subject to $\sum_{a=1}^n r_t(a) = \OPT$. Using a standard Lagrangian calculation, we obtain that the maximizer is when $r_t(a) = \frac{\OPT}{n}$ for all the atoms $a$. Hence, we have the regret in \refeq{eq:thm:naive}. \\

	Now, let us look at the resources. Here, we will critically use the fact that each atom has a dedicated resource. Since, every atom has a dedicated resource, while analyzing a particular resource, we need to consider just the corresponding arm.
	In other words, we have
	\begin{equation*}
	\forall j \in [d] \quad |\sum_{t\leq T} \chi_t(j) - \sum_{t\leq T} \vec{C_t^-(j)}\cdot\vec{x_t}|
	= |\sum_{t\leq T} C_t(a_j, j).Y_t(a_j) - \sum_{t\leq T} C_t^-(a_j, j) x_t(a_j)|
	\end{equation*}

	Here, $a_j$ is the arm dedicated to resource $j$. Since $\{Y_t(a_j) : t\leq T \}$ form a Martingale, we can use the concenteration bounds similar to sampling a single atom and the arguments in \cite{agrawal2014bandits} goes as-is. Hence, with probability $1-ndT\; e^{-\Omega(\alpha)}$ the algorithm does not run out of resources before round $T$.

} 
\section{Applications and special cases}
\label{sec:applications}

Let us discuss some notable examples of \SemiBwK (which generalize some of the numerous applications listed in \cite{BwK-focs13}). Our results for these examples improve exponentially over a naive application of the \BwK framework. Compared to what can be derived from \citep{AgrawalDevanur-ec14,CBwK-nips16}, our results feature a substantially better dependence on parameters, a much better per-round running time, and apply to a wider range of parameters. However, we leave open the possibility that the regret bounds can be improved for some special cases. 

\OMIT{A more detailed discussion, including framing the examples in the \SemiBwK framework and comparisons to prior work, can be found in the supplement, see Section~\ref{sec:moreApplications}.}

\xhdr{Dynamic pricing.}
The dynamic pricing application is as follows. The algorithm has $d$ products on sale with limited supply: for simplicity, $B$ units of each.
Following \citet{BesbesZeevi-or12}, we allow supply constraints \emph{across} products, \eg a ``gadget" that goes into multiple products. In each round $t$, an agent arrives (who can buy any subset of the products), the algorithm chooses a vector of prices $p_t\in [0,1]^d$ to offer the agent, and the agent chooses what to buy at these prices. For simplicity, the agent is interested in buying (or is only allowed to buy) at most one item of each product.  The agent has a valuation vector over products, so that the agent buys a given product if and only if her valuation for this product is at least as high as the offered price. The entire valuation vector is drawn as an independent sample from a fixed and unknown distribution (but valuations may be correlated across products). The algorithm maximizes the total revenue from sales.

To side-step discretization issues, we assume that prices are restricted to a known finite subset $S\subset [0,1]$. Achieving general regret bounds without such restriction appears beyond reach of the current techniques for \BwK.%
\footnote{Prior work on dynamic pricing with limited supply
\citep[\eg][]{BZ09,DynPricing-ec12,BwK-focs13}
achieves regret bounds without restricting itself to a particular finite set of prices, but only for a simple special case of (essentially) a single product.}

To model it as a \SemiBwK problem, the set of atoms is all price-product pairs. The combinatorial constraint is that at most one price is chosen for each product. (If an action does not specify a price for some product, the default price is used.) This is a ``partition matroid" constraint, see Appendix~\ref{partitionMatroid}. Rewards correspond to revenue from sales, and resources correspond to inventory constraints.

We obtain regret
    $\tilde{O}(d \sqrt{dB|S|} + \sqrt{T|S|})$
using Corollary~\ref{maincor},
whenever
    $B> \tilde{\Omega}(n+\sqrt{nT})$.
This is because $\OPT \leq dB$, since that is the maximum number of products available, and the number of atoms is $n=d|S|$.
        
For comparison, results of \citep{AgrawalDevanur-ec14,CBwK-nips16} apply only when
    $B>\sqrt{n}\, T^{3/4}$,
and yield regret bound of
    $\tilde{O}(d^3 |S|^2\sqrt{T})$.%
\footnote{We obtain this by plugging in $\OPT \leq dB$ and $n=d|S|$ into their regret bound. For dynamic pricing the total per-resource consumption is bounded by $1$, so we can apply their results without rescaling the consumption.}
Thus, our regret bounds feature a better dependence on the number of allowed prices $|S|$ (which can be very large) and the number of products $d$. Further, our regret bounds hold in a meaningful way for the much larger range of values for budget $B$.

For a naive application of the \BwK framework, arms correspond to every possible realization of prices for the $d$ products. Thus, we have $|S|^d$ arms, with a corresponding exponential blow-up in regret. 

\xhdr{Dynamic assortment.}
The dynamic assortment problem is similar to dynamic pricing in that the algorithm is selling $d$ products to an agent, with a limited inventory $B$ of each product, and is interested in maximizing the total revenue from sales. As before, agents can have arbitrary valuation vectors, drawn from a fixed but unknown distribution. However, the algorithm chooses which products to offer, whereas all prices are fixed externally. There is a large number of products to choose from, and any subset of $k\ll d$ of them can be offered in any given round. 

To model this as \SemiBwK, atoms correspond to products, and actions correspond to subsets of at most $k$ atoms. The combinatorial constraint forms a matroid (see Appendix~\ref{partitionMatroid}). Rewards correspond to sales, and resources correspond to products, as in dynamic pricing. Since $\OPT \leq \min(dB, kT)$, Corollary~\ref{maincor} yields regret
     $\tilde{O}(k \sqrt{dT})$ when $B>\Omega(T)$,
and regret
    $\tilde{O}(d \sqrt{dB} + \sqrt{dT})$
in general. 

In a naive application of \BwK, arms are subsets of $k$ products. Hence, we have $O(d^k)$ arms. The other parameters of the problem remain the same. This leads to regret bound $\tilde{O}(d\sqrt{Bd^k})$, with an exponential dependence on $k$.

\OMIT{Since resources and atoms correspond to products, the extension in Theorem~\ref{thm:naive} applies, too. It yields regret bound
    $\tilde{O}(d\sqrt{dB})$,
which is better when $B\ll T/d^2$.}

\xhdr{Repeated auctions.}
Consider a repeated auction with adjustable parameters, \eg repeated second-price auction with reserve price that can be adjusted from one round to another. While prior work \citep{RepeatedAuctions-soda13,BwK-focs13} concerned running one repeated auction, we generalize this scenario to multiple repeated auctions with shared inventory (\eg the same inventory may be sold via multiple channels to different audiences).

More formally, the auctioneer is running $r$ simultaneous repeated auctions to sell a shared inventory of $d$ products, with limited supply $B$ of each product (\eg different auctions can cater to different audiences). Each auction has a parameter which the algorithm can adjust over time. We assume that this parameter comes from a finite domain $S \subset [0,1]$. For simplicity, assume the auctions are synchronized with one another. As in prior work, we assume that in every round of each auction a fresh set of participants arrives, sampled independently from a fixed joint distribution, and only a minimal feedback is observed: the products sold and the combined revenue.

Following prior work \citep{RepeatedAuctions-soda13,BwK-focs13}, we only assume minimal feedback: for each auction, what were the products sold and what was the combined revenue from this auction. In particular, we do not assume that the algorithm has access to participants' bids. Not using participants' bids is desirable for privacy considerations, and in order to reduce the participants' incentives to game the learning algorithm.

To model this problem as \SemiBwK, atoms are all auction-parameter pairs. The combinatorial constraint is that an action must specify at most one  parameter value for each auction.  This corresponds to partition matroid constraints, see Appendix~\ref{partitionMatroid}. There is a ``default parameter" for each auction, in case an action does not specify the parameter. We have a resource for each product being auctioned. For simplicity, each product has supply of $B$. Note that $\OPT \leq dB$ and number of atoms is $n=r|S|$. Hence, our main result yields regret
    $\tilde{O}(d \sqrt{r|S|B} + \sqrt{r|S|T})$.

A naive application of the \BwK framework would have arms that correspond to all possible combinations of parameters, for the total of $O(|S|^r)$ arms. Again, we have an exponential blow-up in regret. Alternatively, one may try running $r$ seperate instances of BwK, one for each auction, but that may result result in budgets being violated since the items are shared across the auctions and it is unclear a priori how much of each item will be sold in each auction.

One can also consider a ``flipped" version of the previous example, where the algorithm is a bidder rather than the auction maker. The bidder participates  in $r$ repeated auctions, \eg ad auctions for different keywords. We assume a stationary environment: bidder's utility from a given bid in a given round of a given auction is an independent sample from a fixed but unknown distribution. The only limited resource here is the bidder's budget $B$.  Bids are constrained to lie in a finite subset $S$. 

To model this as \SemiBwK, atoms correspond to the auction-bid pairs. The combinatorial constraint is that each action must specify at most one bid for each auction. (There is a ``default bid" for each auction in case an action does not specify the bid for this auction.) There is exactly one resource, which is money and the total budget is $B$. Note that the number of atoms is $n=r|S|$. Hence, our main result yields regret
$\tilde{O}(\OPT \sqrt{r|S|/B} + \sqrt{r|S|T})$.

A naive application of \BwK would have arms that correspond to all possible combinations of bids, for the total of $O(|S|^r)$ arms; so we have an exponential blow-up in regret.

\section{Numerical Simulations}
\label{sec:simulations}
		
	\begin{figure*}[t]
    	\includegraphics[width=\linewidth]{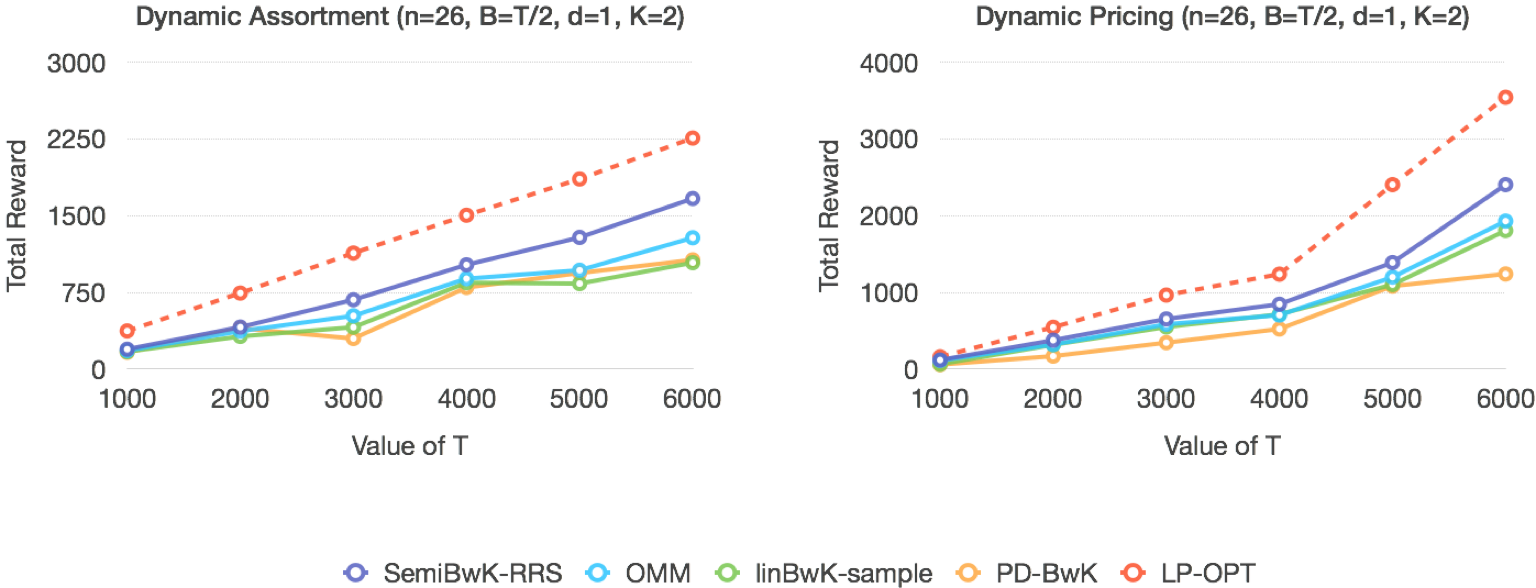}\par
		\caption{Dynamic Assortment (left) and Dynamic Pricing (right) experiments for $n=26$.}
		\label{fig:mainPlots}
	\end{figure*}
		
	\begin{figure*}[h]
		\begin{multicols}{2}
    	\includegraphics[width=\linewidth]{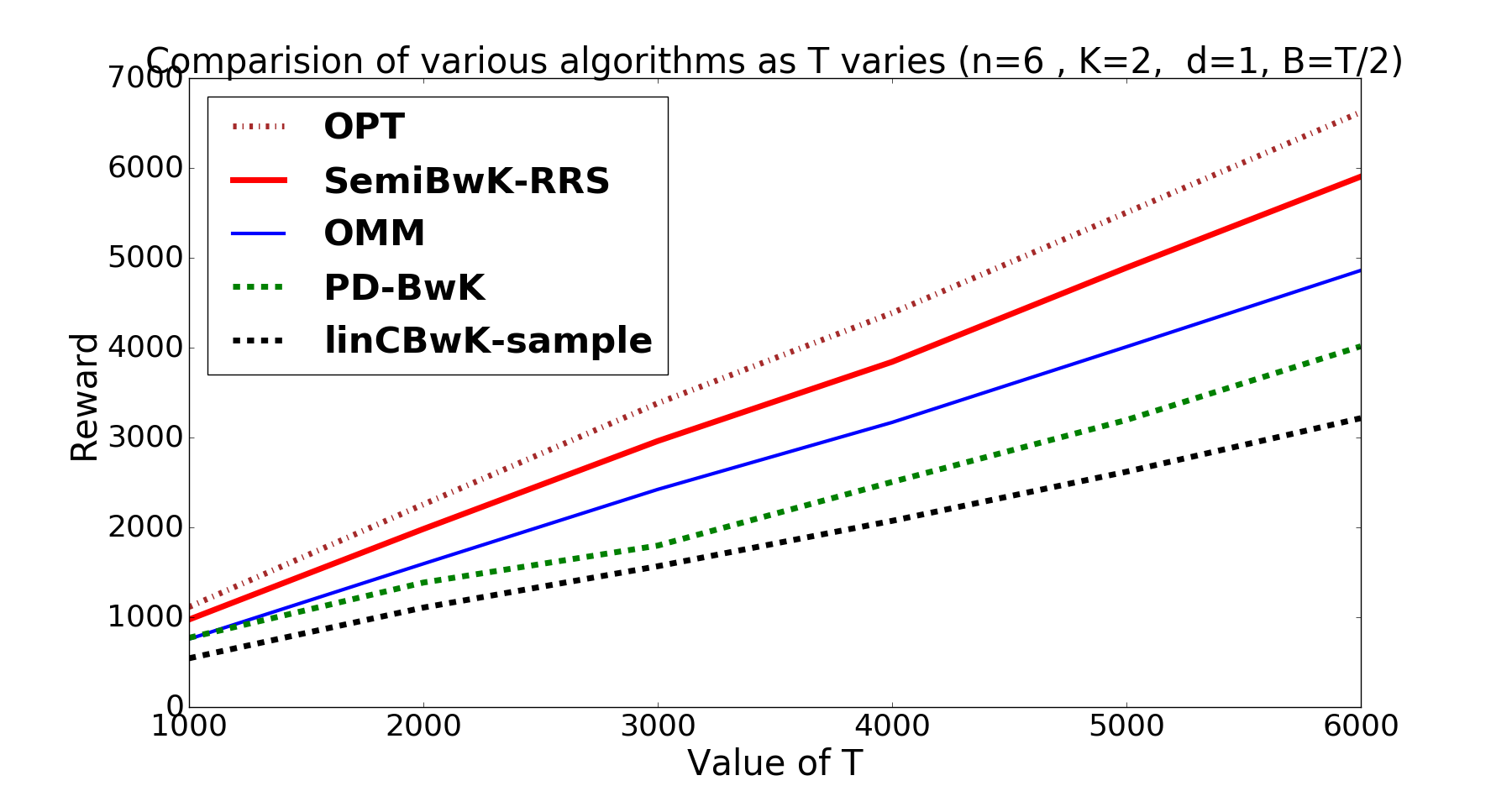}\par
    	\includegraphics[width=\linewidth]{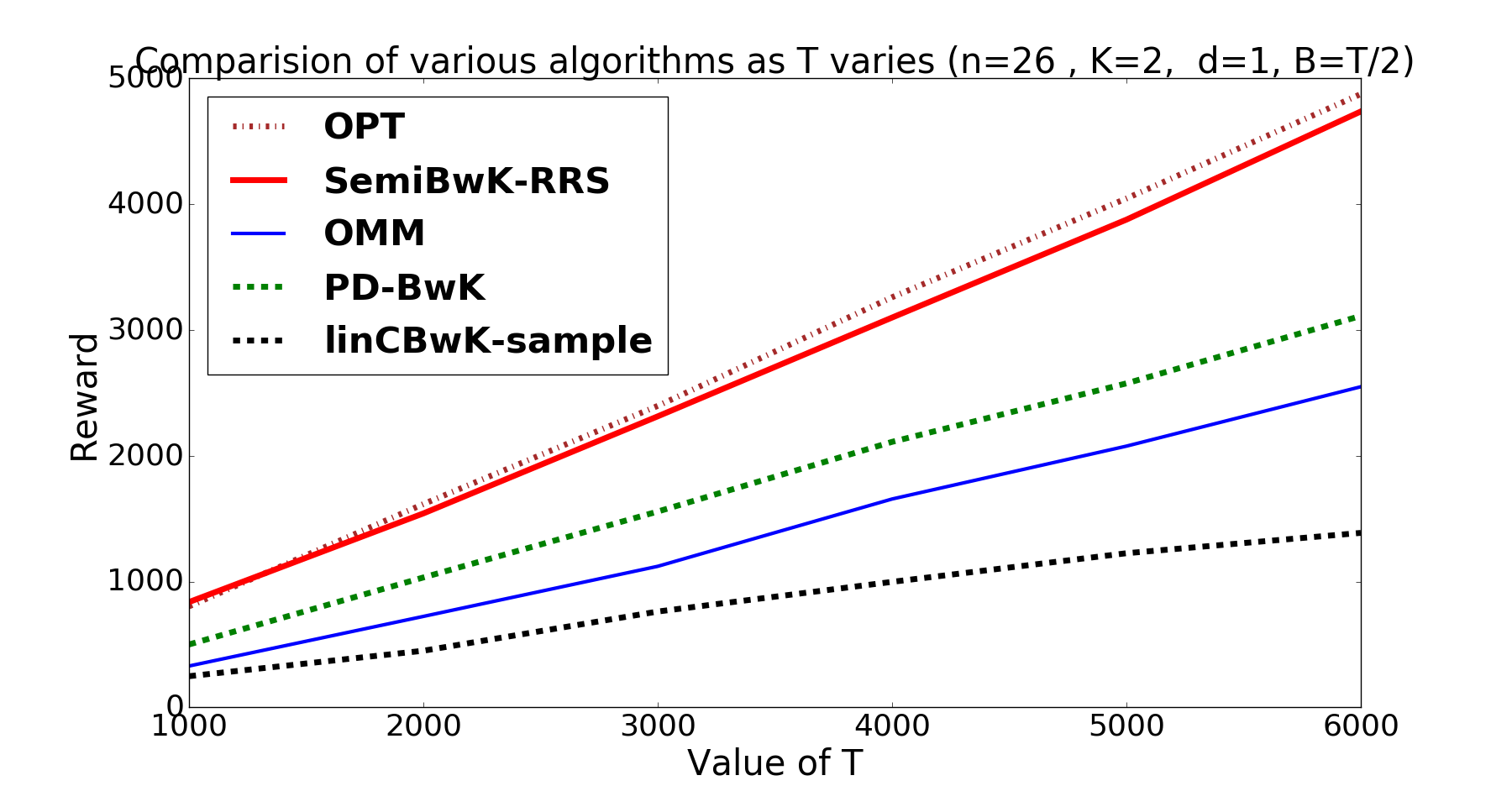}\par
    	\end{multicols}
		\begin{multicols}{2}
    	\includegraphics[width=\linewidth]{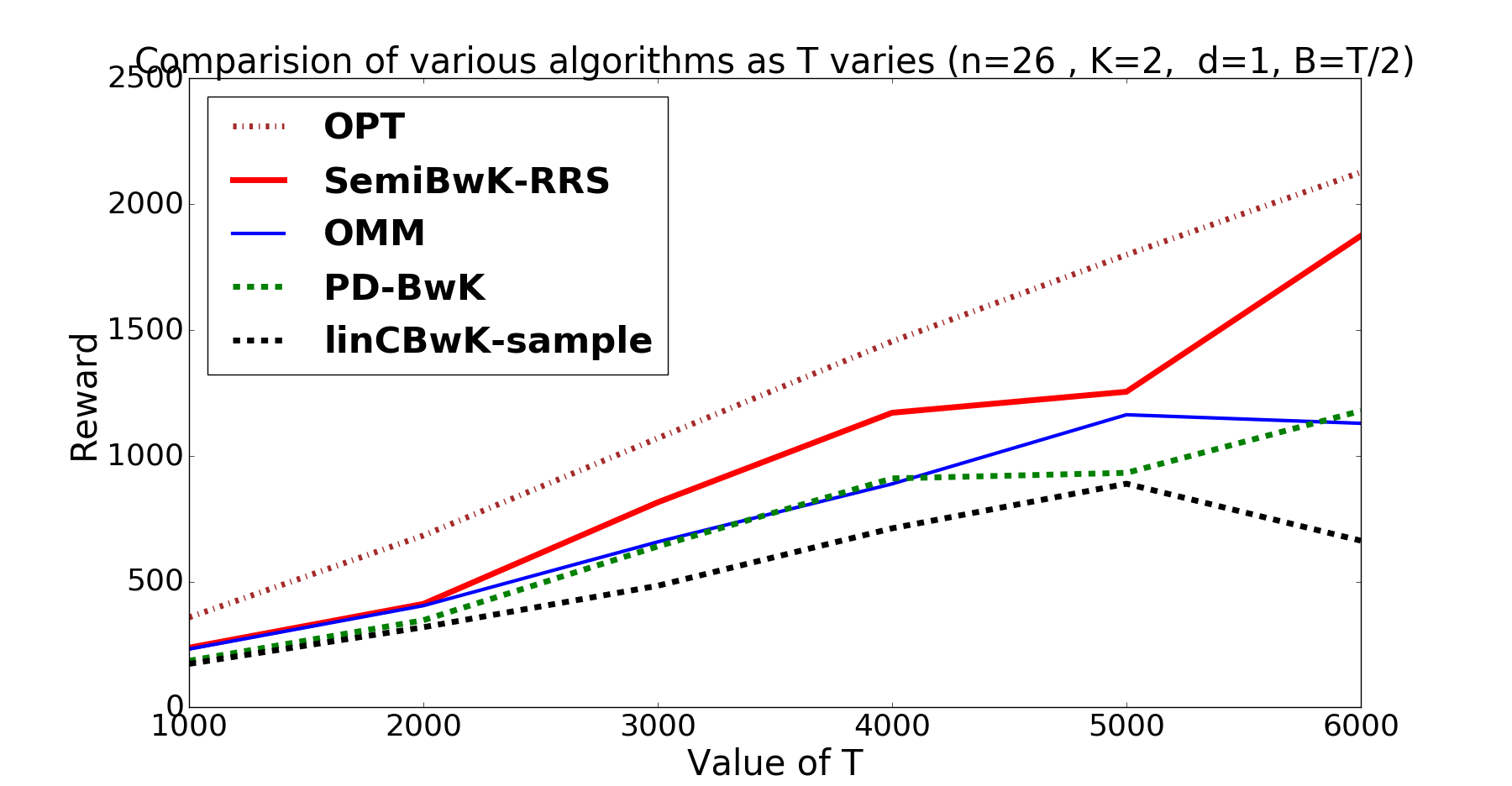}\par
    	\includegraphics[width=\linewidth]{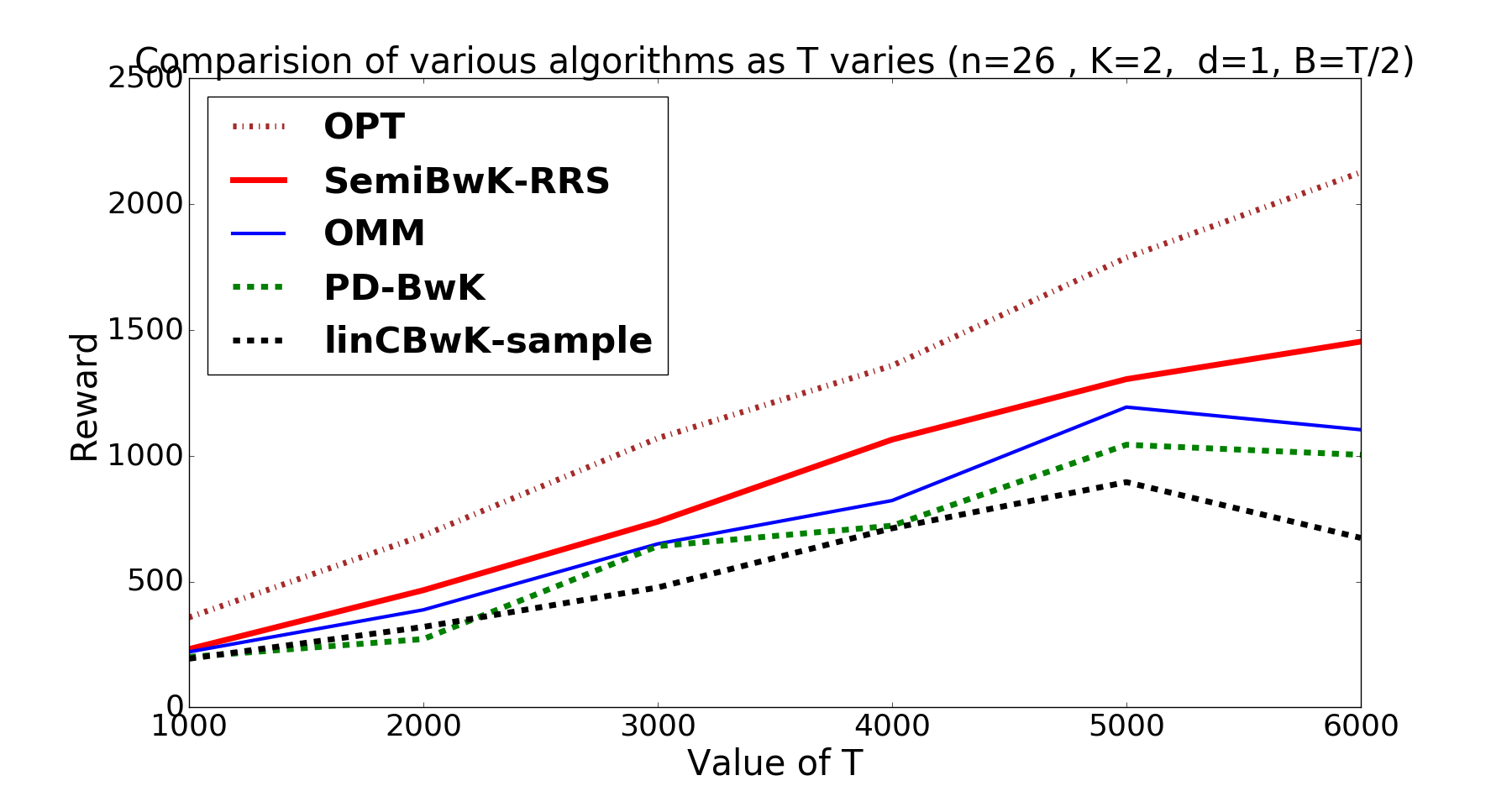}\par
		\end{multicols}
		\caption{Experimental Results for Uniform matroid (left plots) and Partition matroid (right plots) on independent (upper) and correlated (lower) instances for $n=26$.}
		\label{fig:modifiedExamples-large}
	\end{figure*}
	
	\begin{figure*}[h]
		\begin{multicols}{2}
    	\includegraphics[width=\linewidth]{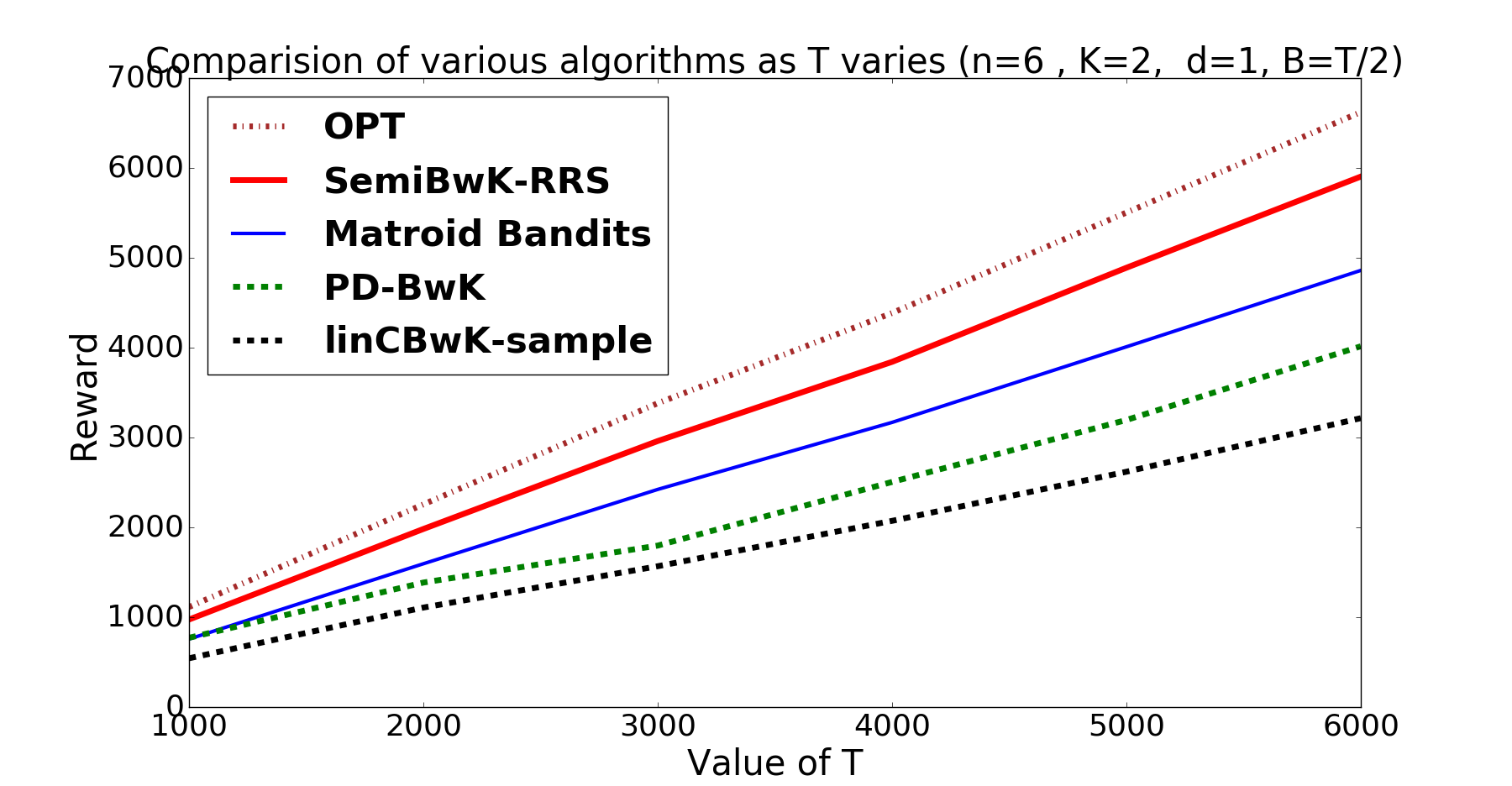}\par
    	\includegraphics[width=\linewidth]{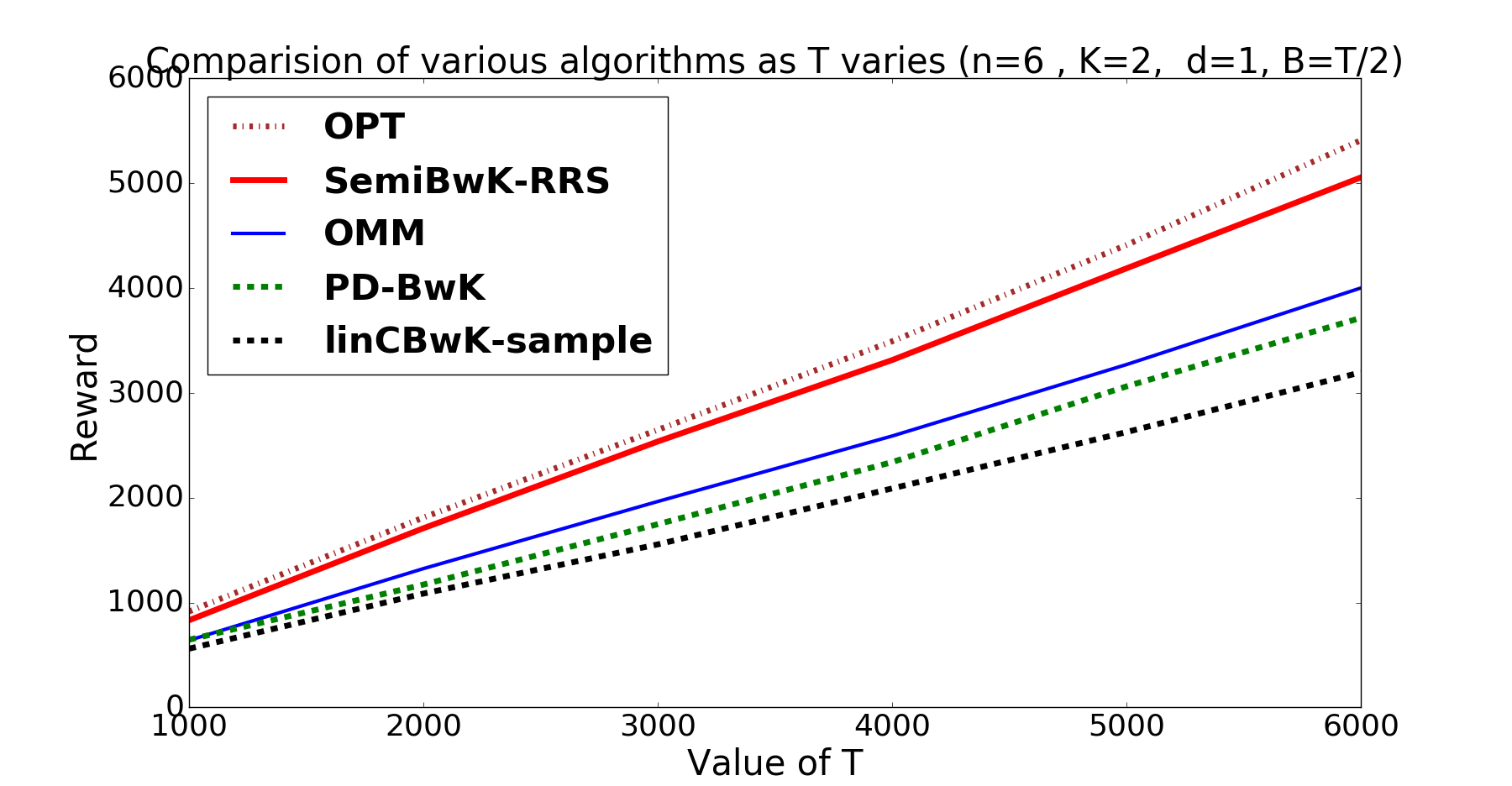}\par
    	\end{multicols}
		\begin{multicols}{2}
    	\includegraphics[width=\linewidth]{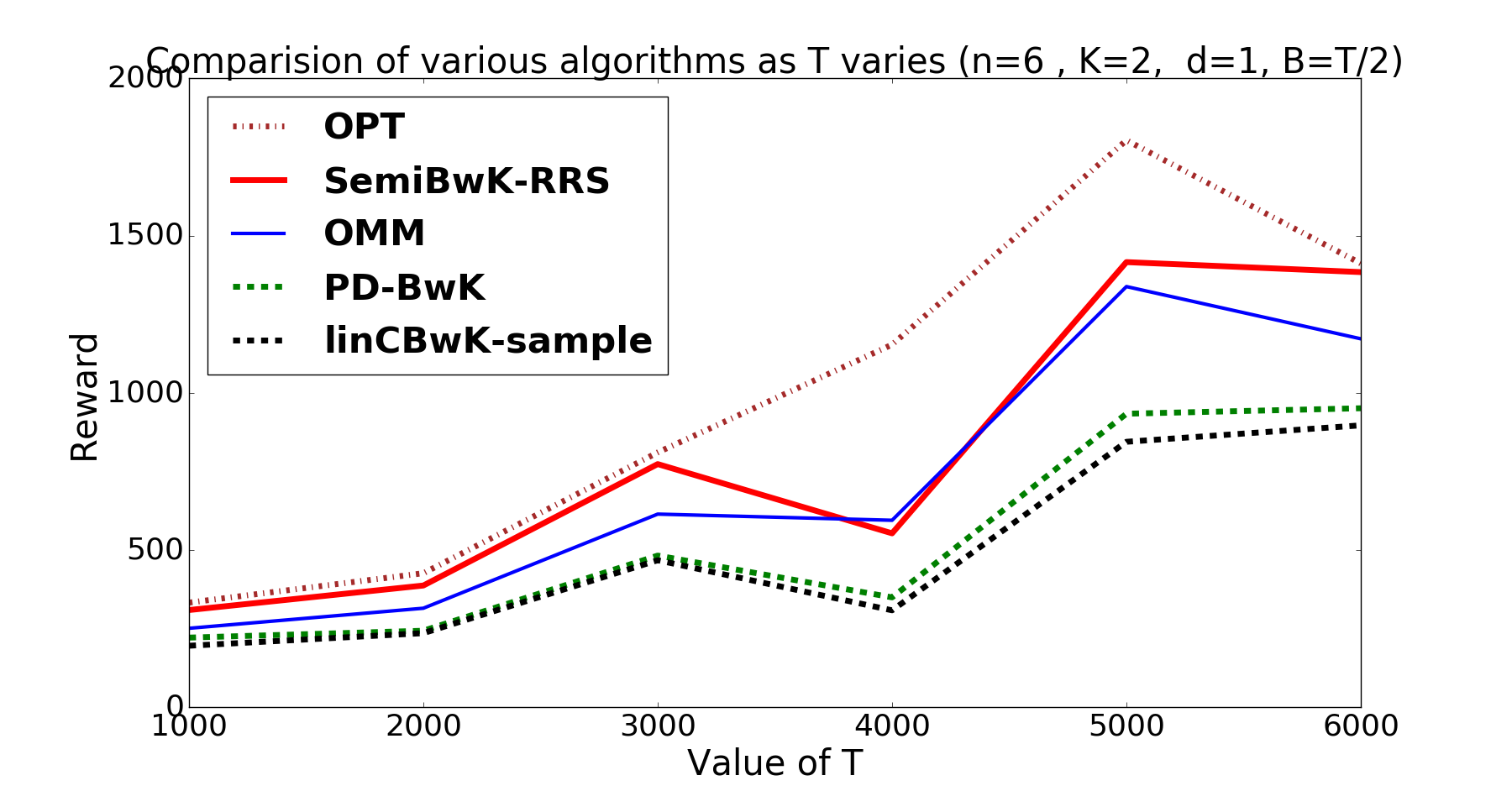}\par
    	\includegraphics[width=\linewidth]{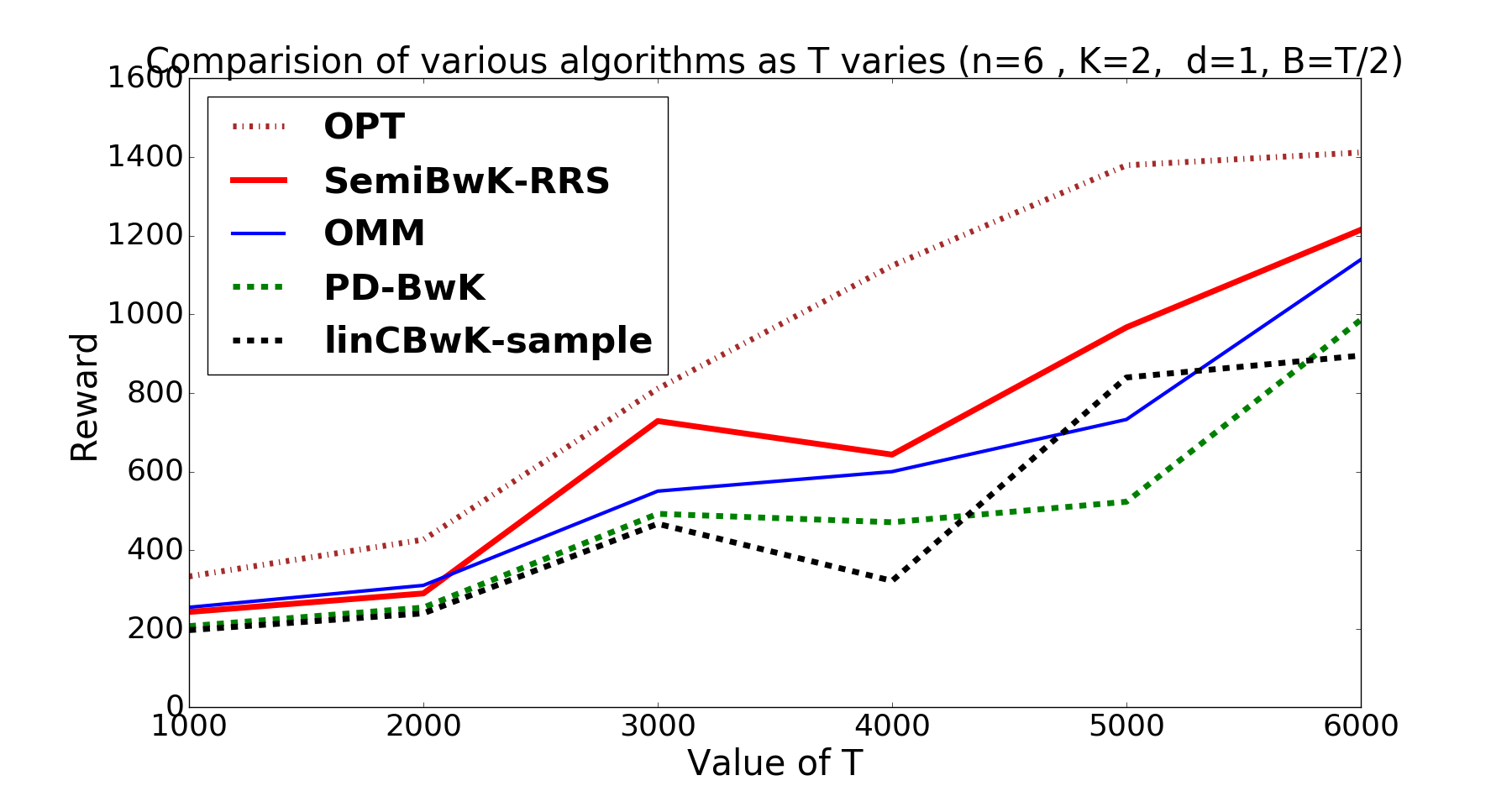}\par
		\end{multicols}
		\caption{Experimental Results for Uniform matroid (left plots) and partition matroid (right plots) on independent (upper) and correlated (lower) instances for $n=6$.}
		\label{fig:modifiedExamples-small}
	\end{figure*}

We ran some experiments on simulated datasets in order to compare our algorithm, \semiBwKUCB, with some prior work that can be used to solve \SemiBwK:
\begin{itemize}
\item the primal-dual algorithm for \BwK from \cite{BwK-focs13},  denoted \pdBwK.

\item an algorithm for combinatorial semi-bandits with a matroid constraint: ``Optimistic Matroid Maximization" from \cite{matroidBandit}, denoted \MB.

\item the linear-contextual \BwK algorithm from \citet{CBwK-nips16}, discussed in the Introduction, denoted \AD.%

\end{itemize}

To speed up the computation in \AD, we used a heuristic modification suggested by the authors in a private communication. This modification did not substantially affect average rewards in our preliminary experiments. We also made a heuristic improvement to our algorithm, setting $\eps =0$ and $\alpha = 5$. We use the same value of $\alpha$ for the \pdBwK algorithm as well.

\xhdr{Problem instances.}
We did not attempt to comprehensively cover the huge variety of problem instances in \SemiBwK. Instead, we focus on two representative applications from Section~\ref{sec:applications}.

The first experiment is on dynamic assortment. We have $n$ products, and for each product $i$ there is an atom $i$ and a resource $i$. The (fixed) price for each product is generated as an independent sample from $U_{[0,1]}$, a uniform distribution on $[0, 1]$. At each round, we sample the buyers's valuation from $U_{[0,1]}$, independently for each product. If the valuation for a given product is greater than the price, one item of this product is sold (and then the reward for atom $i$ is the price, and consumption of resource $i$ is $1$). Else, we set reward for atom $i$ and consumption for resource $i$ to be $0$.

The second experiment is on dynamic pricing with two products. We have $n/2$ allowed prices, uniformly spaced in the $[0,1]$ interval. Recall that atoms correspond to price-product pairs, for the total of $n$ atoms. In each round $t$, the valuation $v_{t,i}$ for each product $i$ is chosen independently from a normal distribution $\mN(v^0_i,1)$ truncated on $[0,1]$. The mean valuation $v^0_i$ is drawn (once for all rounds) from $U_{[0,1]}$. If $v_{t,i}$ is greater than the offered price $p$, one item of this product is sold. Then reward for the corresponding atom $(p,i)$ is the price $p$, and consumption of product $i$ is $1$. If there is no sale for this product, the reward and consumption for each atom $(p,i)$ is set to $0$.

The third experiment is a modification of the dynamic assortment example, in which we ensure that even “non-action” (e.g., no sale) exhausts resources other than time. As in dynamic assortment, we have $n$ products, and for each product $i$ there is an atom $i$ and a resource $i$. The (fixed) price for each product is generated as an independent sample from $U_{[0,1]}$, a uniform distribution on $[0, 1]$. At each round, we sample the buyers's valuation from $U_{[0,1]}$, independently for each product. If the valuation for a given product is greater than the price, one item of this product is sold (and then the reward for atom $i$ is the price, and consumption of resource $i$ is $1$). Else, we do something different from dynamic assortment: we set reward for atom $i$ and consumption for resource $i$ to be the buyer's valuation.

The fourth experiment is a similar modification of the dynamic pricing example. We have $n/2$ allowed prices, uniformly spaced in the $[0,1]$ interval. Recall that atoms correspond to price-product pairs, for the total of $n$ atoms. In each round $t$, the valuation $v_{t,i}$ for each product $i$ is chosen independently from a normal distribution $\mN(v^0_i,1)$ truncated on $[0,1]$. The mean valuation $v^0_i$ is drawn (once for all rounds) from $U_{[0,1]}$. If the valuation for a given product $i$ is greater than the offered price $p$, one item of this product is sold (and then reward for the corresponding atom $(p,i)$ is the price, and consumption of product $i$ is $1$). If there is no sale for this product, we do something different from dynamic pricing. For each atom $(p,i)$, if $p<v_{t,i}$ then the reward for atom $(p,i)$ is drawn independently from $U_{[0,1]}$ and resource consumption is $1$; else, reward is $0$ and consumption is $.3$.
While dynamic assortment is modeled with a uniform matroid, and dynamic pricing is modeled with a partition matroid, we tried both matroids on each family.

\xhdr{Experimental setup and results.} We choose various values of $n$, $B$ and $T$ and run our algorithms on the above two datasets assuming both a uniform matroid constraint and a partition matroid constraint. We choose $n \in \{ 6, 26 \}$, $T \in \{1000, 2000, 3000, 4000, 5000, 6000\}$ and  $B = T/2$. The maximum number of atoms in any action is set to $K=2$. For a given algorithm, dataset and configuration of $n$ and $T$, we simulate each algorithm for $20$ independent runs and take the average. We calculate the total reward obtained by the algorithm at the end of $T$ time-steps.

Figure~\ref{fig:mainPlots} shows results for the first two experiments. Figures~\ref{fig:modifiedExamples-large} and~\ref{fig:modifiedExamples-small} show the results on the third and fourth experiments. Our algorithm achieves the best regret among the competitors. As a benchmark, we included the performance of the fractional allocation in $\LP_{\OPT}$, denoted $\OPT$.

         \begin{figure}[!t]
 		 \centering
    	\includegraphics[width=0.48\textwidth]{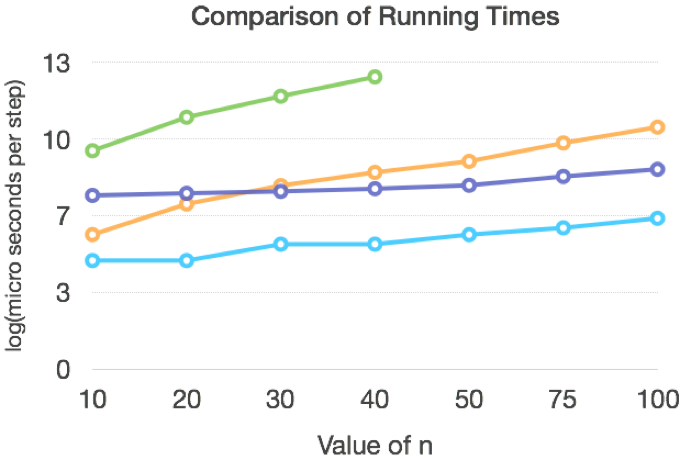}
 		 \caption{Variation of per-step running times as $n$ increases for the various algorithms.}
 		 \label{fig:runTimes}
		\end{figure}

\xhdr{Additional experiment.} \AD and \pdBwK have running times proportional to the number of actions. We ran an additional experiment which compared per-step running times. We first calculate the average running time for every $10$ steps and take the median of $50$ such runs. For both Uniform matroid and Partition matroid, we run the faster RRS due to \cite{gandhi2006dependent}. See Figure~\ref{fig:runTimes} for results.

\xhdr{Details of heuristic implementation of \AD.} We now briefly describe the heuristic we use to simulate the \AD algorithm. Note that even though the per-time-step running time of \AD is reasonable, it takes a significant time when we want to perform simulations for many time-steps. The time-consuming step in the \AD algorithm is the solution to a convex program for computing the optimistic estimates (namely $\vec{\tilde{\mu}_t}$ and $\vec{\tilde{W}_t}$). Hence, the heuristic gives a faster way to obtain this estimate. We sample multiple times from a multi-variate Gaussian with mean $\vec{\hat{\mu}}$ and covariance $\vec{M_t}$ (to obtain estimate $\vec{\tilde{\mu}_t}$) and with mean $\vec{\hat{w}_{tj}}$ and covariance $\vec{M}_t$ (to obtain estimate $\vec{\tilde{w}}_{tj}$ for each resource $j$). We use these samples to compute the objective to choose the action at time-step $t$. For each sample, we compute the best action based on the objective in \AD. We finally choose the action that occurs majority number of times in these samples. The number of samples we choose is set to 30.

\xhdr{Language Details of algorithms.} All algorithms except \AD were implemented in Python. The \AD algorithm was implemented in MATLAB. This difference is crucial when we compare running times since language construct can speed-up or slow down algorithms in practice. However, it is known that \footnote{https://www.mathworks.com/products/matlab/matlab-vs-python.html} for matrix operations commonly encountered in engineering and statistics, MATLAB implementations runs several orders of magnitude faster than the corresponding python implementation. Since \AD is the slowest of the four algorithms, our comparison of running times across languages is justified.

\xhdr{Acknowledgements.} Karthik would like to thank Aravind Srinivasan for some useful discussions.

\bibliographystyle{abbrvnat}
\bibliography{bib-abbrv,refs,bib-slivkins,bib-AGT,bib-random,bib-bandits}

\newpage\pagebreak

\appendix
\section{Proof of Theorem in Preliminaries}
\label{app:prob}

Theorem~\ref{thm:prelim:additive} follows easily from Theorem 3.3 in \citet{IKChernoff}. 

\OMIT{
	\xhdr{Proof of Theorem \ref{prelim:radTheorem}}
	
	For part (a), the proof is very similar to that by \cite{LipschitzMAB-merged-arxiv} and others. We will split it into two cases.
	\begin{itemize}
		\item
		$\mu \geq \frac{\alpha}{6m}$:
		Invoking Theorem \ref{thm:prelim:negChernoff}, Equation (a) with $\epsilon = \frac{1}{2} \sqrt{\frac{\alpha}{6m\mu}}$ we have that with probability
		$1-e^{-\Omega(\alpha)}$ that $|X - \mu| \leq \epsilon \mu$. Note that $\epsilon \mu \leq \frac{\mu}{2}$. Hence,
		
		\begin{align*}
		|X - \mu| &< \frac{1}{2} \sqrt{\frac{\alpha \mu}{6m}} &\\
		&\leq \frac{1}{2} \sqrt{\frac{2 \alpha X}{6m}} & \text{Since $|X - \mu| \leq \mu/2$}\\
		&\leq \rad_{\alpha}(X , m)&
		\end{align*}
		\item
		$\mu \leq \frac{\alpha}{6m}$:
		Invoke Theorem \ref{thm:prelim:negChernoff} Equation (b) with $a = \frac{\alpha}{m}$. Then with probability $1-O(2^{-\alpha})$ we have $X < \frac{\alpha}{m}$. Therefore,
		\[|X - \mu| \leq \frac{\alpha}{m} \leq \rad_{\alpha}(X, m) \]
	\end{itemize}
	
	Finally, similar to \cite{LipschitzMAB-arxiv}, \cite{BabaioffBS13}, \cite{BwK-full} we have that $|X - \mu| \leq \rad_{\alpha}(X , m)$ implies that $\rad_{\alpha}(X , m) \leq 3\rad_{\alpha}(\mu , m)$.\\

	To show part (b) of the Theorem, consider the following. Again we will have two cases.
	
	\begin{itemize}
		\item $\hat{\mu} \geq \frac{\alpha}{6 m}$: Let $\epsilon = \frac{1}{2} \sqrt{\frac{\alpha \hat{\mu}}{6m}}$ in Eq.~\ref{eq:thm2.4}. Note that, we have with probability $1-\exp(-\Omega(\alpha))$, $|X - \hat{\mu}| \leq \epsilon = \frac{1}{2} \sqrt{\frac{\alpha \hat{\mu}}{6m}} \leq \frac{\hat{\mu}}{2}$. Last inequality is because $\hat{\mu} \geq \frac{\alpha}{6 m}$. Hence, $|X - \hat{\mu}| \leq \epsilon = \frac{1}{2} \sqrt{\frac{\alpha \hat{\mu}}{6m}} \leq \rad_{\alpha}(X, m)$. The last inequality follows from similar arguments used in proving part (a) of the theorem.
		\item $\hat{\mu} \leq \frac{\alpha}{6 m}$: Note that this implies $\hat{\mu} - X \leq \frac{\alpha}{6 m} \leq \frac{\alpha}{m} \leq \rad_{\alpha}(X, m)$. Hence, we only need to show that $X - \hat{\mu} \leq \rad_{\alpha}(X, m)$.
		Note that from Theorem 2.1 in \cite{BwK-full}, we have that with probability at least $1-\exp[-\Omega(b)]$,
		
		\[
		X - \hat{\mu} \leq X - \frac{1}{m} \sum_{t=1}^m \mathbb{E}[X_t|X_1, X_2, \ldots, X_{t-1}] \leq \rad_{\alpha}(X, m)
		\]
		
		The first inequality is from the condition that $\hat{\mu} \geq \frac{1}{m} \sum_{t=1}^m \mathbb{E}[X_t|X_1, X_2, \ldots, X_{t-1}]$.  Therefore, we have $|X - \hat{\mu}| \leq \rad_{\alpha}(X, m)$ which implies $\rad_{\alpha}(X, m) \leq 3 \rad_{\alpha}(\hat{\mu}, m)$.
		
		This completes the proof of part (b) of Theorem~\ref{prelim:radTheorem}.
		
	\end{itemize}
} 

\begin{theorem*}[Theorem~\ref{thm:prelim:additive}]
Let $\mX = (X_1, X_2, \ldots, X_m)$ denote a collection of random variables which take values in $[0,1]$, and let
    $X := \frac{1}{m} \sum_{i=1}^m X_i$
be their average. Suppose $\mX$ satisfies \eqref{eq:neg-corr-half}, \ie
$\mathbb{E}[\prod_{i \in S}X_i] \leq (\tfrac12)^{|S|}$
 for every $S \subseteq [m]$. Then for some absolute constant $c$,
	\begin{align}\label{eq:thm:prelim:additive-supp}
	\Pr[X \geq \tfrac12 + \eta] \leq c \cdot e^{-2m\eta^2}
	\qquad (\forall \eta>0).
	\end{align}
\end{theorem*}

\begin{proof}
Fix $\eta>0$. From Theorem 3.3 in \citet{IKChernoff}, we have that
\[ \Pr[X \geq \tfrac12 + \eta] \leq c\cdot  e^{-m \KL{1/2 + \eta}{1/2}},\]
where $\KL{\cdot}{\cdot}$ denotes KL-divergence, so that
\begin{align}\label{eq:thm:prelim:additive-supp-KL}
\KL{\tfrac12 + \eta}{\tfrac12}
    = (\tfrac12 + \eta) \log(1 + 2\eta) +
      (\tfrac12 - \eta) \log(1- 2\eta).
\end{align}
From Pinsker's inequality we have, $\KL{1/2 + \eta}{1/2}  \geq 2 \eta^2$, which implies \eqref{eq:thm:prelim:additive-supp}.
\end{proof}


\section{Matroid constraints}
\label{appx:matroid}

To make this paper more self-contained, we provide more background on matroid constraints and special cases thereof.

\OMIT{ 
Our main result in Theorem~\ref{mainresult} holds as long as $\mF$ forms a matroid. Thus, we define matroid constraints and list several important special cases thereof.
in Appendices~\ref{appx:constraints-matroid} and ~\ref{appx:constraints-examples}, respectively. The extension in Theorem~\ref{thm:naive} holds for linearizable action sets. We provide an example in Appendix~\ref{appx:constraints-independent-set}, namely: we define constraints specified by independent sets of a graph, and show that they are linearizable.} 


Recall that in \SemiBwK, we have a finite ground set whose elements are called ``atoms", and a family $\mF$ of ``feasible subsets" of the ground set which are the actions.
To be consistent with the literature on matroids, the ground set will be denoted $E$. Family $\mF$ of subsets of $E$ is called a matroid if it satisfies the following properties:

\begin{itemize}
	\item
	\textbf{Empty set}: The empty set $\phi$ is present in $\mF$
	\item
	\textbf{Hereditary property}: For two subsets $X, Y \subseteq E$ such that $X \subseteq Y$, we have that $$Y \in \mF \implies X \in \mF$$
	\item
	\textbf{Exchange property}:
	For $X, Y \in \mF$ and $|X| > |Y|$, we have that
	
	$$\exists e \in X \setminus Y: Y \cup \{e\} \in \mF$$
\end{itemize}

Matroids are \emph{linearizable}, \ie the convex hull of $\mF$ forms a polytope in $\R^E$. (Here subsets of $\mF$ are intepreted as binary vectors in $\R^E$.) In other words, there exists a set of linear constraints whose set of feasible \emph{integral} solutions is $\mF$. In fact, the convex hull of $\mF$, a.k.a. the \emph{matroid polytope}, can be represented via the following linear system:
\begin{equation*}
\begin{array}{ll@{}ll}
\tag{LP-Matroid}
\label{pt:matroid}
x(S) \leq \text{rank}(S) & \forall S \subseteq E \\
x_e \in [0,1]^E & \forall e \in E.
\end{array}
\end{equation*}

\noindent Here
    $x(S) := \sum_{e \in S} x_e$,
and
    $\text{rank}(S) = \max \{ |Y| : Y \subseteq S, Y \in \mF\}$
is the ``rank function" for $\mF$.

$\mF$ is indeed the set of all feasible integral solutions of the above system. This is a standard fact in combinatorial optimization, \eg see Theorem 40.2 and its corollaries in \citet{schrijver2002combinatorial}.


We will now describe some well-studied special cases of matroids. That they indeed are special cases of matroids is well-known, we will not present the corresponding proofs here.

In all LPs presented below, we have variables $x_e$ for each arom $e\in E$, and we use shorthand
    $x(S) := \sum_{e \in S} x_e$
for $S\subset E$.

\xhdr{Cardinality constraints.}
Cardinality constraint is defined as follows: a subset $S$ of atoms belongs to $\mF$ if and only if $|S|\leq K$ for some fixed $K$. This is perhaps the simplest constraint that our results are applicable to. In the context of \SemiBwK, each action selects at most $K$ atoms.

The corresponding induced polytope is as follows:
\begin{equation*}
\begin{array}{ll@{}ll}
\tag{LP-Cardinality}
\label{pt:cardinality}
x(E) \leq K&\\
x_e \in [0,1] & \forall e \in E.
\end{array}
\end{equation*}

\xhdr{Partition matroid constraints.}
\label{partitionMatroid}
A generalization of cardinality constraints, called partition matroid constraints, is defined as follows. Suppose we have a collection $B_1 \LDOTS B_k$ of disjoint subsets of $E$, and numbers $d_1 \LDOTS d_k$. A subset $S$ of atoms belongs to $\mF$ if and only if $|S \cap B_i| \leq d_i$ for every $i$. Partition matroid constraints appear in several applications of \SemiBwK such as dynamic pricing, adjusting repeated auctions, and repeated bidding. In these applications, each action selects one price/bid for each offered product. Also, partition matroid constraints can model clusters of mutually exclusive products in dynamic assortment application.

The induced polytope is as follows:

\begin{equation*}
\begin{array}{ll@{}ll}
\tag{LP-PartitionMatroid}
\label{pt:partitionmatroid}
 x(B_i) \leq d_i & \forall i \in [k] \\
	x_e \in [0,1] & \forall e \in E.
\end{array}
\end{equation*}

\OMIT{
	\xhdr{Matroid Intersection}\\
	There are a set of ground atoms $e_1, e_2, \ldots, e_n$. Our atoms correspond to these ground atoms. We are given two matroids $\mathcal{M}_1$ and $\mathcal{M}_2$. Let $r_1$ and $r_2$ be their corresponding rank functions. In each step, the algorithm is allowed to choose any subset of atoms such that, this subset is an independent set in both $\mathcal{M}_1$ and $\mathcal{M}_2$. As before, denote $x(S) := \sum_{e \in S} x_e$. Formally, the polytope can be written as follows:
	
	\begin{equation*}
	\begin{array}{ll@{}ll}
	\tag{LP-MatroidIntersection}
	\label{pt:matroidint}
	x(S) \leq r_1(S) & \forall S \subseteq E \\
	x(S) \leq r_2(S) & \forall S \subseteq E \\
	x_e \geq 0 & \forall e \in E \\
	\vec{x} \in \mathbb{R}^{E} &
	\end{array}
	\end{equation*}
}

\OMIT{
\xhdr{Graph Matching}\\
For semi-bandit constraints as a general graph matching polytope we are given an underlying graph $G(V, E)$. The edges of this graph correspond to the atoms. In each step, the algorithm can choose any valid matching as the subset. A matching is a subset of edges such that no two edges in this subset share a vertex. The polytope defining a graph matching constraint is as follows:

\begin{equation*}
\begin{array}{ll@{}ll}
\tag{LP-Matching}
\label{pt:matching}
\displaystyle\sum_{e \in \delta(u)}x(e) \leq 1 & \forall u \in U \\
\displaystyle\sum_{e \in E(S)}x(e) \leq k & \forall S \subseteq V \text{ of cardinality } |S| = 2k+1\\
1 \geq x(e) \geq 0 & \forall e \in E \\
\end{array}
\end{equation*}

In the above polytope, $E(S)$ refers to the edges present in the sub-graph induced by the vertices $S$.
}

\xhdr{Spanning tree constraints.}
Spanning tree constraints describe spanning trees in a given undirected graph $G = (V, E)$, where the atoms correspond to edges in the graph. A spanning tree in $G$ is a subset $E'\subset E$ of edges such that $(V,E')$ is a tree. Action set $\mF$ consists of all spanning trees of $G$.

The induced polytope is as follows:

\begin{equation*}
\begin{array}{ll@{}ll}
\tag{LP-SpanningTree}
\label{pt:spanning}
x(E_S) \leq |S|-1 & \forall S \subseteq V \\
x(E_V) = |V|-1 &\\
x_e \in [0,1] & \forall e \in E.
\end{array}
\end{equation*}

Here, $E_S$ denotes the edge set in subgraph induced by node set $S\subset V$.

\OMIT{
\xhdr{Path constraints.}
Path constraints describe paths in a given graph $G = (V, E)$, where the atoms correspond to edges in the graph. Path constraints are commonly used in conjunction with the \emph{online routing problem} (\citep{Bobby-stoc04} and much follow-up work), where each action is a path in the network.

Formally, we are given an undirected graph  $G=(V,E)$ and a source sink pair $(s, t)$ in this graph. Action set $\mF$ consists of all $s$-$t$ paths in $G$.

We index edges as $e=(i,j)$, where $i,j\in V$. The corresponding induced polytope is as follows:

\begin{equation*}
\begin{array}{ll@{}ll}
\tag{LP-Path}
\label{pt:path}
\sum_{j\in V}x_{(i, j)} - x_{(j, i)} = 0 & \forall i \in V \setminus \{s, t\} \\
\sum_{j\in V}x_{(i, j)} - x_{(j, i)} = 1 & i = s \\
\sum_{j\in V}x_{(i, j)} - x_{(j, i)} = -1 & i = t \\
x_e \in [0,1] & \forall e \in E.
\end{array}
\end{equation*}
}

\OMIT{ 
\subsection{Independent set constraints}
\label{appx:constraints-independent-set}

Independent set constraints are specified as follows. Suppose $V$ is the set of atoms, and $G=(V,E)$ is an undirected graph. Action set $\mF$ consists of all independent sets in $G$, \ie all subsets $S\subset V$ such that $(u,v)\not\in E$, for any nodes $u,v\in S$. Independent set constraints arise in the dynamic assortment application with mutually exclusive products.

Independent set constraints are linearizable. The corresponding induced polytope is as follows:
\begin{equation*}
\begin{array}{ll@{}ll}
\tag{LP-IS}
\label{pt:path}
x_u + x_v \leq 1 & \forall (u, v) \in E \\
x_e \in [0,1] &  \forall e \in E.
\end{array}
\end{equation*}
} 

\end{document}